\documentclass{article}

\PassOptionsToPackage{numbers, compress}{natbib}
\usepackage[final]{neurips_2020}

\usepackage[utf8]{inputenc} % allow utf-8 input
\usepackage[T1]{fontenc}    % use 8-bit T1 fonts
\usepackage[hidelinks]{hyperref}       % hyperlinks
\usepackage{url}            % simple URL typesetting
\usepackage{booktabs}       % professional-quality tables
\usepackage{amsfonts}       % blackboard math symbols
\usepackage{nicefrac}       % compact symbols for 1/2, etc.
\usepackage{microtype}      % microtypography

\setcitestyle{open={[},close={]}}

\usepackage{xcolor}

\usepackage{algorithm,algorithmic}
\usepackage{amsmath,amsthm,amssymb,amsfonts,bm,mathtools,braket,units,array,comment,cases,bbm}
\usepackage{accents}
\usepackage{centernot}
\usepackage{filecontents}
\usepackage{color}
\usepackage{enumitem}
\usepackage{tikz}
\usetikzlibrary{matrix}
\usetikzlibrary{shapes.geometric}
\usetikzlibrary{arrows.meta}
\usetikzlibrary{positioning,shapes,arrows}
\usetikzlibrary{calc}
\usetikzlibrary{shapes.multipart}
\usetikzlibrary{patterns}
\usetikzlibrary{bayesnet}
\usepackage{pgfplots}
\pgfplotsset{width=7cm,compat=1.8}

\def\CausalEffect{\mathcal{C}}%\mathrm{CausalEffect}}
\def\DataFidelity{\mathcal{D}}
\def\E{\mathbb{E}}

\def\R{\mathbb{R}}
\def\p{p}

\def\Xhat{\widehat{X}}
\newcommand{\norm}[1]{\left\|#1\right\|}
\newcommand{\ip}[2]{\langle #1, #2 \rangle}
\newcommand{\abs}[1]{\left|#1\right|}
\newcommand{\argmin}{\mathop{\mathrm{argmin}}}

\newtheorem{theorem}{Theorem}
\newtheorem{lemma}[theorem]{Lemma}
\newtheorem{proposition}[theorem]{Proposition}

\newtheorem{definition}[theorem]{Definition}

\definecolor{goldenpoppy}{rgb}{1.0,0.761,0.039}
\definecolor{brightnavyblue}{rgb}{0.047,0.482,0.863}
\definecolor{rossocorsa}{rgb}{0.816,0.000,0.000}

\title{Generative causal explanations \\of black-box classifiers}

\author{%
  Matthew O'Shaughnessy, Gregory Canal, Marissa Connor, \\ \textbf{Mark Davenport, and Christopher Rozell} \\
  School of Electrical \& Computer Engineering\\
  Georgia Institute of Technology
}

\begin{document}

\maketitle

\begin{abstract}
We develop a method for generating causal post-hoc explanations of black-box classifiers based on a learned low-dimensional representation of the data. The explanation is causal in the sense that changing learned latent factors produces a change in the classifier output statistics. To construct these explanations, we design a learning framework that leverages a generative model and information-theoretic measures of causal influence. Our objective function encourages both the generative model to faithfully represent the data distribution and the latent factors to have a large causal influence on the classifier output. Our method learns both global and local explanations, is compatible with any classifier that admits class probabilities and a gradient, and does not require labeled attributes or knowledge of causal structure. Using carefully controlled test cases, we provide intuition that illuminates the function of our objective. We then demonstrate the practical utility of our method on image recognition tasks.\footnote{Code is available at \texttt{https://github.com/siplab-gt/generative-causal-explanations}.}
\end{abstract}

\section{Introduction}
\label{sec:introduction}
There is a growing consensus among researchers, ethicists, and the public that machine learning models deployed in sensitive applications should be able to \emph{explain} their decisions \cite{doshi-velez2017role,kroll2017accountable}. A powerful way to make ``explain'' mathematically precise is to use the language of causality: explanations should identify \emph{causal} relationships between certain data aspects --- features which may or may not be semantically meaningful --- and the classifier output \cite{miller2019explanation,pearl2019seven,moraffah2020causal}. In this conception, an aspect of the data helps explain the classifier if changing that aspect (while holding other data aspects fixed) produces a corresponding change in the classifier output.
%These types of explanations clearly separate the aspects of the data that cause the classifier output from those aspects that are merely correlated with it.

Constructing causal explanations requires reasoning about how changing different aspects of the input data affects the classifier output, but these observed changes are only meaningful if the modified combination of aspects occurs naturally in the dataset. A challenge in constructing causal explanations is therefore the ability to change certain aspects of data samples without leaving the data distribution. In this paper we propose a novel learning-based framework that overcomes this challenge. Our framework has two fundamental components that we argue are necessary to operationalize a causal explanation: a method to \emph{represent and move within the data distribution}, and a \emph{rigorous metric for causal influence} of different data aspects on the classifier output.

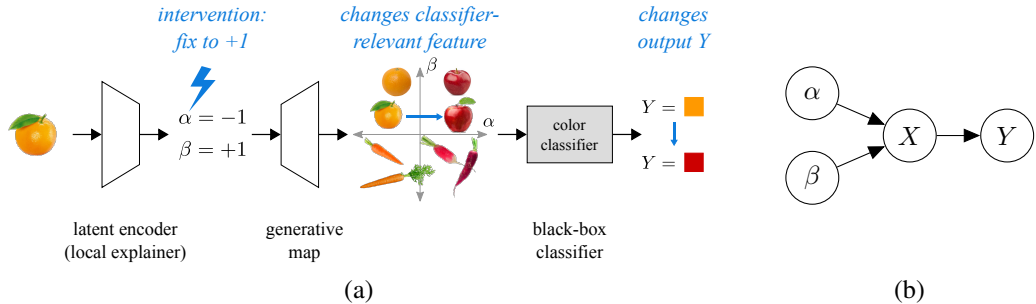
\begin{figure}
    \centering
    \begin{tabular}{ >{\centering\arraybackslash} m{3.7in} >{\centering\arraybackslash} m{0.05in} >{\centering\arraybackslash} m{1.3in} }
        \includegraphics{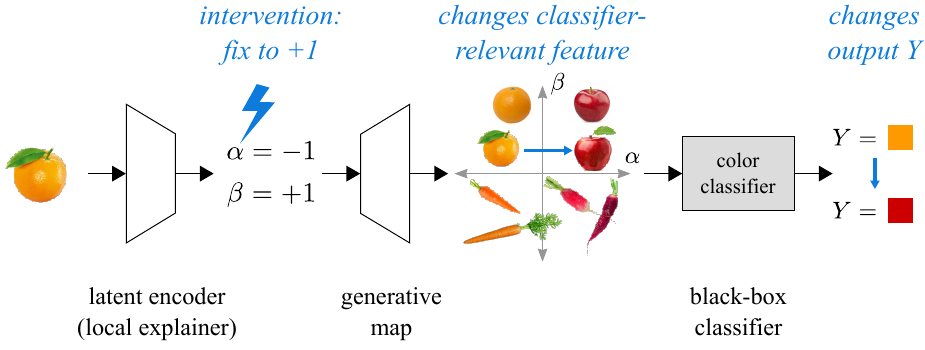} & ~~~ &
        \begin{tikzpicture}
            \node[latent] (X) {$X$};
            \node[latent, left=of X, xshift=1.2em, yshift=1.6em] (alpha) {$\alpha$};
            \node[latent, left=of X, xshift=1.2em, yshift=-1.6em] (beta) {$\beta$};
    	    \node[latent, right=of X, xshift=-1.2em] (Y) {$Y$};
    	    \edge {alpha} {X};
    	    \edge {beta} {X};
    	    \edge {X} {Y};
    	\end{tikzpicture}
	    \\
	    (a) & & (b)
	\end{tabular}
	\caption{(a) Computational architecture used to learn explanations. Here, the low-dimensional representation $(\alpha,\beta)$ learns to describe the color and shape of inputs. Changing $\alpha$ (color) changes the output of the classifier, which detects the color of the data sample, while changing $\beta$ (shape) does not affect the classifier output. (b) DAG describing our causal model, satisfying principles in Section \ref{sec:methods/model}.}
    \label{fig:graphical-abstract-dag}
\end{figure}

To do this, we construct a generative model consisting of a disentangled representation of the data and a generative mapping from this representation to the data space (Figure \ref{fig:graphical-abstract-dag}(a)). We seek to learn this disentangled representation in such a way that each factor controls a different aspect of the data, and a subset of the factors have a large causal influence on the classifier output. To formalize this notion of causal influence, we define a structural causal model (SCM) \cite{pearl2009causality} that relates independent latent factors defining data aspects, the classifier inputs, and the classifier outputs. Leveraging recent work on information-theoretic measures of causal influence \cite{ay2008information,janzing2013quantifying}, we use the independence of latent factors in the SCM to show that in our framework the causal influence of the latent factors on the classifier output can be quantified simply using mutual information. The crux of our approach is an optimization program for learning a mapping from the latent factors to the data space. The objective ensures that the learned disentangled representation represents the data distribution while simultaneously encouraging a subset of latent factors to have a large causal influence on the classifier output.

A natural benefit of our framework is that the learned disentangled representation provides a rich and flexible vocabulary for explanation. This vocabulary can be more expressive than feature selection or saliency map-based explanation methods: a latent factor, in its simplest form, could describe a single feature or mask of features in input space, but it can also describe much more complex patterns and relationships in the data. Crucially, unlike methods that crudely remove features directly in data space, the generative model enables us to construct explanations that respect the data distribution. This is important because an explanation is only meaningful if it describes combinations of data aspects that naturally occur in the dataset. For example, a loan applicant would not appreciate being told that his loan would have been approved if he had made a negative number of late payments, and a doctor would be displeased to learn that her automated diagnosis system depends on a biologically implausible attribute.

Once the disentangled representation is learned, explanations can be constructed using the generative mapping. Our framework can provide both global and local explanations: a practitioner can understand the aspects of the data that are important to the classifier at large by visualizing the effect in data space of changing each causal factor, and they can determine the aspects that dictated the classifier output for a specific input by observing its corresponding latent values. These visualizations can be much more descriptive than saliency maps, particularly in vision applications.

The major contributions of this work are a new conceptual framework for generating explanations using causal modeling and a generative model (Section \ref{sec:methods}), analysis of the framework in a simple setting where we can obtain analytical and intuitive understanding (Section \ref{sec:lingauss}), and a brief evaluation of our method applied to explaining image recognition models (Section \ref{sec:vae}).

\section{Related work}
\label{sec:related-work}
We focus on methods that generate \emph{post-hoc} explanations of black-box classifiers. While post-hoc explanations are typically categorized as either global (explaining the entire classifier mechanism) or local (explaining the classification of a particular datapoint) \cite{guidotti2018survey}, our framework joins a smaller group of methods that globally learn a model that can be then used to generate local explanations \cite{dabkowski2017real,chen2018learning,bang2019explaining,schwab2019cxplain}.

\textbf{Forms of explanation.} Post-hoc explanations come in varying forms. Some methods learn an interpretable model such as a decision tree that \emph{approximates the black-box} either globally \cite{craven1996extracting,bastani2017interpretability,guo2018explaining} or locally \cite{ribeiro2016why,lakkaraju2017interpretable,kim2019learning,wagner2019interpretable}. A larger class of methods create local explanations directly in the data space, performing \emph{feature selection} or creating \emph{saliency maps} using classifier gradients \cite{simonyan2013deep,bach2015pixelwise,shrikumar2017learning,sundararajan2017axiomatic,lundberg2017unified} or by training a new model \cite{dabkowski2017real}. A third category of methods generate \emph{counterfactual data points} that describe how inputs would need to be altered to produce a different classifier output \cite{ribeiro2018anchors,white2019providing,zhang2018interpreting,wachter2018counterfactual,carter2019what,mothilal2020explaining,vanlooveren2020interpretable}. Other techniques identify the \emph{points in the training set} most responsible for a particular classifier output \cite{koh2017understanding,khanna2019interpreting}. Our framework belongs to a separate class of methods whose explanations consist of a low-dimensional set of \emph{latent factors} that describe different aspects (or ``concepts'') of the data. These latent factors form a rich and flexible vocabulary for both global and local explanations, and provide a means to represent the data distribution. Unlike some methods that learn concepts using labeled attributes \cite{kim2018interpretability,parafita2019explaining}, we do not require side information defining data aspects; rather, we visualize the learned aspects using a generative mapping to the data space as in \cite{li2018deep,harradon2018causal,alvarezmelis2018robust}. This type of latent factor explanation has also been used in the construction of self-explaining neural networks \cite{li2018deep,al-shedivat2018contextual}.
% check categorization of lundberg2017unified - their framework includes LIME which does not require gradients?
% counterfactual methods that respect data manifold - \cite{mothilal2020explaining,vanlooveren2020interpretable}

\textbf{Causality in explanation.} Because explanation methods seek to answer ``why'' and ``how'' questions that use the language of cause and effect \cite{miller2019explanation,pearl2019seven}, causal reasoning has played an increasingly important role in designing explanation frameworks \cite{moraffah2020causal}. (For similar reasons, causality has played a prominent part in designing metrics for fairness in machine learning \cite{kusner2017counterfactual,kilbertus2017avoiding, zhang2018fairness, zhang2018equality, wu2019pc}.) Prior work has quantified the impact of features in data space by using Granger causality \cite{schwab2019cxplain}, a priori known causal structure \cite{frye2019asymmetric,parafita2019explaining}, an average or individual causal effect metric \cite{chattopadhyay2019neural,kim2019learning}, or by applying random valued-interventions \cite{datta2016algorithmic}. Other work generates causal explanations by performing interventions in different network layers \cite{narendra2018explaining}, using latent factors built into a modified network architecture \cite{harradon2018causal}, or using labeled examples of human-interpretable latent factors \cite{goyal2020explaining}.
% todo - add "Other approaches seek global interpretations of classes, such as by studying components of images that causally influence other components \cite{lopez2017discovering}." ?

Generative models have been used to compute interventions that respect the data distribution \cite{alvarez-melis2017causal,parafita2019explaining,kim2019learning,chang2019explaining}, a key idea in this paper. Our work, however, is most similar to methods using generative models whose explanations use notions of causality and are constructed directly from latent factors. \citeauthor{goyal2020explaining} compute the average causal effect (ACE) of human-interpretable concepts on the classifier \cite{goyal2020explaining}, but require labeled examples of the concepts and suffer from limitations of the ACE metric \cite{janzing2013quantifying}. \citeauthor{harradon2018causal} construct explanations based on latent factors, but these explanations are specific to neural network classifiers and require knowledge of the classifier network architecture \cite{harradon2018causal}. Our method is unique in constructing a framework from principles of causality that generates latent factor-based explanations of black-box classifiers without requiring side information.
%We adapt information-theoretic methods for quantifying causal effect (studied in \cite{ay2008information,janzing2013quantifying,schamberg2019direct}), allowing us to more fully capture complex causal interaction compared to the metrics used by many existing methods. We use this metric to construct an objective that encourages latent factors to both respect the data manifold and have a large causal effect on the classifier statistics.
% other source to potentially add:
% - hooker2019benchmark (proposes "remove and retrain (ROAR)")
% - grover2019uncertainty (http://proceedings.mlr.press/v89/grover19a.html)

\paragraph{Disentanglement perspective.} Our method can also be interpreted as a \emph{disentanglement} procedure \cite{bengio2013representation,higgins2018definition} supervised by classifier output probabilities. Unlike work that encourages a one-to-one correspondence between individual latent factors and semantically meaningful features (i.e., ``data generating factors''), we aim to separate the latent factors that are relevant to the classifier's decision from those that are irrelevant. We outline connections to this literature in more detail in Section \ref{sec:methods/disentanglement}.

\section{Methods}
\label{sec:methods}
Our goal is to explain a black-box classifier $f \colon \mathcal{X} \to \mathcal{Y}$ that takes data samples $X \in \mathcal{X}$ and assigns a probability to each class $Y \in \{1, \dots, M\}$ (i.e., $\mathcal{Y}$ is the $M$-dimensional probability simplex). We assume that the classifier also provides the gradient of each class probability with respect to the classifier input.

Our explanations take the form of a low-dimensional and independent set of ``causal factors'' $\alpha \in \R^K$ that, when changed, produce a corresponding change in the classifier output statistics. We also allow for additional independent latent factors $\beta \in \R^L$ that contribute to representing the data distribution but need not have a causal influence on the classifier output. Together, $(\alpha,\beta)$ constitute a low-dimensional representation of the data distribution $p(X)$ through the generative mapping $g \colon \R^{K+L} \to \mathcal{X}$. The generative mapping is learned so that the explanatory factors $\alpha$ have a large causal influence on $Y$, while $\alpha$ and $\beta$ together faithfully represent the data distribution (i.e., $\p(g(\alpha,\beta)) \approx \p(X)$). The $\alpha$ learned in this manner can be interpreted as aspects \emph{causing} $f$ to make classification decisions \cite{pearl2009causality}.

To learn a generative mapping with these characteristics, we need to define (i) a model of the causal relationship between $\alpha$, $\beta$, $X$, and $Y$, (ii) a metric to quantify the causal influence of $\alpha$ on $Y$, and (iii) a learning framework that maximizes this influence while ensuring that $\p(g(\alpha,\beta)) \approx \p(X)$.

\subsection{Causal model}
\label{sec:methods/model}

We first define a directed acyclic graph (DAG) describing the relationship between $(\alpha,\beta)$, $X$, and $Y$, which will allow us to derive a metric of causal influence of $\alpha$ on $Y$. We propose the following principles for selecting this DAG:
\vspace{-0.05in}
\begin{enumerate}[leftmargin=2em]
    \item[\textbf{(1)}] \textbf{The DAG should describe the functional (causal) structure of the data, not simply the statistical (correlative) structure.} This principle allows us to interpret the DAG as a structural causal model (SCM) \cite{pearl2009causality} and interpret our explanations causally.
    \item[\textbf{(2)}] \textbf{The explanation should be derived from the classifier output $Y$, not the ground truth classes.} This principle affirms that we seek to understand the action of the classifier, not the ground truth classes.
    \item[\textbf{(3)}] \textbf{The DAG should contain a (potentially indirect) causal link from $X$ to $Y$.} This principle ensures that our causal model adheres to the functional operation of $f \colon X \to Y$.
\end{enumerate}
\vspace{-0.05in}

Based on these principles, we adopt the DAG shown in Figure \ref{fig:graphical-abstract-dag}(b). Note that the difference in the roles played by $\alpha$ and $\beta$ is subtle and not apparent from the DAG alone: the difference arises from the fact that the functional relationship defining the causal connection $X \to Y$ is $f$, which by construction uses only features of $X$ that are controlled by $\alpha$. In other words, interventions on both $\alpha$ and $\beta$ produce changes in $X$, but only interventions on $\alpha$ produce changes in $Y$. A key feature of this DAG is that the latent factors $(\alpha,\beta)$ are independent, which we enforce with an isotropic prior when learning the generative mapping. This independence improves the parsimony and interpretability of the learned disentangled representation (see Appendix \ref{sec:supp/causal-obj-variants}). It also results in our metric for causal influence simplifying to mutual information. Importantly, unlike methods that \emph{assume} independence of features in data space (e.g., \cite{datta2016algorithmic,ribeiro2016why,shrikumar2017learning,lundberg2017unified}), our framework \emph{intentionally learns} independent latent factors.

\subsection{Metric for causal influence}
\label{sec:methods/causal-influence}

We now derive a metric $\CausalEffect(\alpha,Y)$ for the causal influence of $\alpha$ on $Y$ using the DAG in Figure \ref{fig:graphical-abstract-dag}(b). A satisfactory measure of causal influence in our application should satisfy the following principles:
\vspace{-0.05in}
\begin{enumerate}[leftmargin=2em]
    \item[\textbf{(1)}] \textbf{The metric should completely capture functional dependencies.} This principle allows us to capture the complete causal influence of $\alpha$ on $Y$ through the generative mapping $g$ and classifier $f$, which may both be defined by complex and nonlinear functions such as neural networks.
    \item[\textbf{(2)}] \textbf{The metric should quantify indirect causal relationships between variables.} This principle allows us to quantify the indirect causal relationship between $\alpha$ and $Y$.
\end{enumerate}
\vspace{-0.05in}
Principle 1 eliminates common metrics such as the average causal effect (ACE) \cite{holland1988causal} and analysis of variance (ANOVA) \cite{lewontin1974analysis}, which capture only causal relationships between first- and second-order statistics, respectively \cite{janzing2013quantifying}. Recent work has overcome these limitations by using information-theoretic measures \cite{ay2008information,janzing2013quantifying,schamberg2018quantifying}. Of these, we select the \emph{information flow} measure of \cite{ay2008information} to satisfy Principle 2 because it is node-based, naturally accommodating our goal of quantifying the causal influence of $\alpha$ on $Y$.

The information flow metric adapts the concept of mutual information typically used to quantify \emph{statistical} influence to quantify \emph{causal} influence by the observational distributions in the standard definition of conditional mutual information with interventional distributions:
\begin{definition}[Ay and Polani 2008 \cite{ay2008information}]
    Let $U$ and $V$ be disjoint subsets of nodes. The \emph{information flow from $U$ to $V$} is
    \begin{equation}
        I(U \to V) \coloneqq \int_U p(u) \int_V p(v \mid do(u)) \log \frac{p(v \mid do(u))}{\int_{u'} p(u') p(v \mid do(u')) du'} dV dU,
        \label{eq:information-flow}
    \end{equation}
    where $do(u)$ represents an intervention in a causal model that fixes $u$ to a specified value regardless of the values of its parents \cite{pearl2009causality}.
\end{definition}

The independence of $(\alpha,\beta)$ makes it simple to show that information flow and mutual information coincide in our DAG:
\begin{proposition}[Information flow in our DAG]
    \label{prop:information-flow}
    The information flow from $\alpha$ to $Y$ in the DAG of Figure \ref{fig:graphical-abstract-dag}(b) coincides with the mutual information between $\alpha$ and $Y$. That is, $I(\alpha \to Y) = I(\alpha ; Y)$, where mutual information is defined as $I(\alpha ; Y) = \E_{\alpha,Y} \left[ \log \frac{p(\alpha,Y)}{p(\alpha) p(Y)} \right]$.
\end{proposition}
The proof, which follows easily from the rules of do-calculus \cite[Thm.~3.4.1]{pearl2009causality}, is provided in Appendix \ref{sec:supp/proofs/information-flow}. Based on this result, we use
\begin{equation}
    \label{eq:causal-effect}
    \CausalEffect(\alpha,Y) = I(\alpha ; Y)
\end{equation}
to quantify the causal influence of $\alpha$ on $Y$. This metric, derived in our work from principles of causality using the DAG in Figure \ref{fig:graphical-abstract-dag}(b), has also been used to select informative features in other work on explanation \cite{gao2016variational,chen2018learning,al-shedivat2018contextual,kanehira2019learning,chang2019game,adel2018discovering}. Our framework, then, generates explanations that benefit from both causal and information-theoretic perspectives. Note, however, that the validity of the causal interpretation is predicated on our modeling decisions; mutual information is in general a correlational, not causal, metric.

Other variants of (conditional) mutual information are also compatible with our development. These variants retain causal interpretations, but produce explanations of a slightly different character. For example, $\sum_{i=1}^K I(\alpha_i; Y)$ and $I(\alpha ; Y \mid \beta)$ (the latter corresponding to the information flow of $\alpha$ on $Y$ when ``imposing'' $\beta$ in \cite{ay2008information}) encourage interactions between the explanatory features to generate $X$. These variants are described and analyzed in more detail in Appendices \ref{sec:supp/causal-obj-variants} and \ref{sec:supp/lingauss-details}.

%Proposition \ref{prop:information-flow} inspires several possible candidates for our causal influence metric $\CausalEffect$, differing based on (a) whether we average the information flows $I(\alpha_i \to Y)$ or simply consider the information flow $I(\alpha \to Y)$, and (b) whether we impose (condition on) the values of the remaining latent factors or not. Each candidate objective gives rise to explanations with causal interpretations, but the \emph{character} of the resulting explanations is subtly different. In Appendix \ref{sec:supp/causal-obj-variants} we provide a proposition that precisely defines the difference between these candidate objectives in information-theoretic terms, and provide a complete justification for our choice. Here, we use
%\begin{equation}
%    \label{eq:causal-effect}
%    \CausalEffect = I(\alpha ; Y).
%\end{equation}
%for two reasons:
%\begin{enumerate}
%    \item Compared to objective variant $\widetilde{\CausalEffect} = I(\alpha \to Y)$, \eqref{eq:causal-effect} has the effect of adding $\sum_{i=1}^K I(Y ; \alpha_{\{1, \dots, K\} \setminus i} \mid \alpha_i)$ to the objective. This term encourages TODO, which we consider desirable.
%    \item Compared to objective variant $\widetilde{\CausalEffect} = I(\alpha \to Y \mid \beta)$, \eqref{eq:causal-effect} has the effect of not adding $I(\alpha ; \beta \mid Y)$ to the objective. This term would encourage TODO, which we consider undesirable.
%\end{enumerate}
%See Appendix \ref{sec:supp/causal-obj-variants} for more details.

\subsection{Optimization framework}
\label{sec:methods/objective}

We now turn to our goal of learning a generative mapping $g \colon (\alpha,\beta) \to X$ such that $p(g(\alpha,\beta)) \approx p(X)$, the $(\alpha,\beta)$ are independent, and $\alpha$ has a large causal influence on $Y$. We do so by solving
\begin{equation}
    \label{eq:objective}
    \underset{g \in G}{\arg\max}\quad \CausalEffect(\alpha, Y) + \lambda \cdot \DataFidelity\left(\p(g(\alpha,\beta)), \p(X)\right),
\end{equation}
where $g$ is a function in some class $G$, $\CausalEffect(\alpha,Y)$ is our metric for the causal influence of $\alpha$ on $Y$ from \eqref{eq:causal-effect}, and $\DataFidelity(p(g(\alpha,\beta)),p(X))$ is a measure of the similarity between $p(g(\alpha,\beta))$ and $p(X)$.

The use of $\DataFidelity$ is a crucial feature of our framework because it forces $g$ to produce samples that are in the data distribution $p(X)$. Without this property, the learned causal factors could specify combinations of aspects that do not occur in the dataset, providing little value for explanation. The specific form of $\DataFidelity$ is dependent on the class of decoder models $G$. In this paper we focus on two specific instantiations of $G$. Section \ref{sec:lingauss} takes $G$ to be the set of linear mappings with Gaussian additive noise, using negative KL divergence for $\DataFidelity$. This setting allows us to provide more rigorous intuition for our model. Section \ref{sec:vae} adopts the variational autoencoder (VAE) framework shown in Figure \ref{fig:graphical-abstract-dag}(a), parameterizing $G$ by a neural network and using a variational lower bound \cite{kingma2014autoencoding} as $\DataFidelity$.

\subsection{Training procedure}
\label{sec:methods/training-procedure}

In practice, we maximize the objective \eqref{eq:objective} using Adam \cite{kingma2017adam}, computing a sample-based estimate of $\CausalEffect$ at each iteration. The sampling procedure is detailed in Appendix \ref{sec:supp/sampling}. Training our causal explanatory model requires selecting $K$ and $L$, which define the number of latent factors, and $\lambda$, which trades between causal influence and data fidelity in our objective. A proper selection of these parameters should set $\lambda$ sufficiently large so that the distributions $p(X \mid \alpha, \beta)$ used to visualize explanations lie in the data distribution $p(X)$, but not so high that the causal influence term is overwhelmed.

\begin{algorithm}[t]
\caption{Principled procedure for selecting $(K,L,\lambda)$.}
\label{alg:parameter-selection}
\begin{algorithmic}[1]
    %\REQUIRE require
    \STATE Initialize $K, L, \lambda = 0$. Optimizing only $\DataFidelity$, increase $L$ until objective plateaus.% Denote this objective value by $d^*$.
    \STATE \textbf{repeat}~ increment $K$ and decrement $L$. Increase $\lambda$ until $\DataFidelity$ approaches value from Step 1.
    \STATE \textbf{until}~ {$\CausalEffect$ reaches plateau. Use $(K,L,\lambda)$ from immediately before plateau was reached.}
\end{algorithmic}
\end{algorithm}

To properly navigate this trade-off it is instructive to view \eqref{eq:objective} as a constrained problem \cite{boyd2004convex} in which $\CausalEffect$ is maximized subject to an upper bound on $\DataFidelity$. Algorithm \ref{alg:parameter-selection} provides a principled method for parameter selection based on this idea. Step 1 selects the total number of latent factors needed to adequately represent $p(X)$ using only noncausal factors. Steps 2-3 then incrementally convert noncausal factors into causal factors until the total explanatory value of the causal factors (quantified by $\CausalEffect$) plateaus. Because changing $K$ and $L$ affects the relative weights of the causal influence and data fidelity terms, $\lambda$ should be increased after each increment to ensure that the learned representation continues to satisfy the data fidelity constraint.% Alternatively, $\lambda$ could be selected using domain knowledge that prescribes an upper bound on $\DataFidelity$ (e.g., an acceptable log-likelihood when $G$ is a set of linear models, or an acceptable mean squared error \cite{higgins2016betavae} when using a VAE architecture).

\subsection{Disentanglement perspective}
\label{sec:methods/disentanglement}

Disentanglement procedures seek to learn low-dimensional data representations in which latent factors correspond to data aspects that concisely and independently describe high dimensional data \cite{bengio2013representation,higgins2018definition}. Although some techniques perform unsupervised disentanglement \cite{higgins2017betavae,chen2016infogan,kim2018disentangling}, it is common to use side information as a supervisory signal.

Because our goal is explanation, our main objective is to separate classifier-relevant and classifier-irrelevant aspects. Our framework can be thought of as a disentanglement procedure with two distinguishing features:

First, we use classifier probabilities to aid disentanglement. This is similar in spirit to disentanglement methods that incorporate grouping or class labels as side information by modifying the VAE training procedure \cite{kulkarni2015deep}, probability model \cite{bouchacourt2018multilevel}, or loss function \cite{ridgeway2018learning}. Although these methods could be adapted for explanation using classifier-based groupings, our method intelligently uses classifier \emph{probabilities} and gradients.

Second, we develop our framework from a causal perspective. Suter et al.\ also develop a disentanglement procedure from principles of causality \cite{suter2019robustly}, casting the disentanglement task as learning latent factors that correspond to parent-less causes in the generative structural causal model. Unlike this framework, we assume that the latent factors are independent based on properties of the VAE evidence lower bound. We then use this fact to show that the commonly-used MI metric measures \emph{causal} influence of $\alpha$ on $Y$ using the information flow metric of \cite{ay2008information}. 

This provides a causal interpretation for information-based disentanglement methods such as InfoGAN \cite{chen2016infogan} (which adds a term similar to $I(\alpha;X)$ to the VAE objective). Encouragement of independence in latent factors plays an important role in much work on disentanglement (e.g., \cite{higgins2017betavae,chen2016infogan,chen2018isolating}); priors that better encourage independence could be applied in our framework to increase the validity of our proposed causal graph.

\section{Analysis with linear-Gaussian generative map}
\label{sec:lingauss}
We first consider the instructive setting in which a linear generative mapping is used to explain simple classifiers with decision boundaries defined by hyperplanes. This setting admits geometric intuition and basic analysis that illuminates the function of our objective.

In this section we define the data distribution as isotropic normal in $\R^N$, $X \sim \mathcal{N}(0,I)$ (but note that elsewhere in the paper we make no assumptions on the data distribution). Let $(\alpha,\beta) \sim \mathcal{N}(0,I)$, and consider the following generative model to be used for constructing explanations:
\begin{equation*}
    g(\alpha,\beta) = \begin{bmatrix} W_{\alpha} & W_{\beta} \end{bmatrix} \begin{bmatrix} \alpha \\ \beta \end{bmatrix} + \varepsilon,
\end{equation*}
where $W_{\alpha} \in \R^{N \times K}$, $W_{\beta} \in \R^{N \times L}$, and $\varepsilon \sim \mathcal{N}(0,\gamma I)$. We illustrate the behavior of our method applied with this generative model on two simple binary classifiers ($Y \in \{0,1\}$).

\begin{figure}
    \centering
    \begin{tabular}{ccccc}

\begin{tikzpicture}[scale=0.25,shading angle from/.style args={line from #1 to #2}{insert path={let \p1=($#2-#1$),\n1={atan2(\y1,\x1)} in},shading angle=\n1}]
    \pgfmathsetmacro{\THETACLASS}{0} % angle of classifier normal from ([1;0])
    \pgfmathsetmacro{\THETAALPHA}{30} % angle of what_alpha from classifier
    \pgfmathsetmacro{\THETABETA}{135} % angle of what_beta from classifier
    \fill[draw=none, fill=brightnavyblue, opacity=0.4] (-4.5,-4.5) -- (-4.5,4.5) -- (0,4.5) -- (0,-4.5) -- cycle;
    \fill[draw=none, fill=goldenpoppy, opacity=0.4] (0,4.5) -- (4.5,4.5) -- (4.5,-4.5) -- (0,-4.5) -- cycle;
    \draw[help lines, color=gray!30] (-4.9,-4.9) grid (4.9,4.9);
    \draw[<->,thick,color=gray!65] (-5,0)--(5,0) node[right]{};
    \draw[<->,thick,color=gray!65] (0,-5)--(0,5) node[above]{};
    \draw[->,thick] (0,0)--({2.5*cos(\THETACLASS+\THETAALPHA)},{2.5*sin(\THETACLASS+\THETAALPHA)}) node[above right, xshift=-0.3em, yshift=-0.3em]{\scalebox{1.0}{$w_\alpha$}};
    \draw[->,thick] (0,0)--({2.5*cos(\THETACLASS+\THETABETA)},{2.5*sin(\THETACLASS+\THETABETA)}) node[above left, xshift=0.3em]{\scalebox{1.0}{$w_\beta$}};
    \draw[->,thick] (0,0)--(2.5,0) node[right, xshift=-0.4em, yshift=-0.4em]{\scalebox{1.0}{$a$}};
    \foreach \x in {-3,-1.5,0,1.5,3}
        {\draw[color=gray, rotate around={\THETACLASS+\THETABETA:({\x*cos(\THETACLASS+\THETAALPHA)},{\x*sin(\THETACLASS+\THETAALPHA)})}] ({\x*cos(\THETACLASS+\THETAALPHA)},{\x*sin(\THETACLASS+\THETAALPHA)}) ellipse (2.5 and 1.25);}
    \draw[{Latex[length=1mm,width=1mm]}-] (2.1,-2.0) -- (2.8,-3) node[below, yshift=3]{\scalebox{1.0}{$p(\Xhat \mid \alpha)$}};
\end{tikzpicture} & 

\begin{tikzpicture}[scale=0.25,shading angle from/.style args={line from #1 to #2}{insert path={let \p1=($#2-#1$),\n1={atan2(\y1,\x1)} in},shading angle=\n1}]
    \pgfmathsetmacro{\THETACLASS}{0} % angle of classifier normal from ([1;0])
    \pgfmathsetmacro{\THETAALPHA}{0} % angle of what_alpha from classifier
    \pgfmathsetmacro{\THETABETA}{90} % angle of what_beta from classifier
    \fill[draw=none, fill=brightnavyblue, opacity=0.4] (-4.5,-4.5) -- (-4.5,4.5) -- (0,4.5) -- (0,-4.5) -- cycle;
    \fill[draw=none, fill=goldenpoppy, opacity=0.4] (0,4.5) -- (4.5,4.5) -- (4.5,-4.5) -- (0,-4.5) -- cycle;
    \draw[help lines, color=gray!30] (-4.9,-4.9) grid (4.9,4.9);
    \draw[<->,thick,color=gray!85] (-5,0)--(5,0) node[right]{};
    \draw[<->,thick,color=gray!85] (0,-5)--(0,5) node[above]{};
    %\draw[dashed,ultra thick] ({-5*cos(\THETACLASS+90)},{-5*sin(\THETACLASS+90)})--({5*cos(\THETACLASS+90)},{5*sin(\THETACLASS+90)}) node[right]{};
    %\draw[-,thick] (-4.5,4.5)--(4.5,-4.5) node[below]{$\alpha$};
    %\draw[-,thick] (-5,-5)--(5,5) node[right]{$\beta$};
    %\draw[->,thick] (0,0)--({3*cos(\THETACLASS)},{3*sin(\THETACLASS)}) node[below right, xshift=-0.3em]{\scalebox{1.0}{$a$}};
    \draw[->,thick] (0,0)--({2.5*cos(\THETACLASS+\THETAALPHA)},{2.5*sin(\THETACLASS+\THETAALPHA)}) node[below, xshift=-0.1em, yshift=-0.2em]{\scalebox{1.0}{$w_\alpha^* \propto a$}};
    \draw[->,thick] (0,0)--({2.5*cos(\THETACLASS+\THETABETA)},{2.5*sin(\THETACLASS+\THETABETA)}) node[above right, xshift=-0.2em, yshift=-0.4em]{\scalebox{1.0}{$w_\beta^*$}};
\end{tikzpicture} & \qquad &

\begin{tikzpicture}[scale=0.25,shading angle from/.style args={line from #1 to #2}{insert path={let \p1=($#2-#1$),\n1={atan2(\y1,\x1)} in},shading angle=\n1}]
    \pgfmathsetmacro{\THETACLASS}{0} % angle of classifier normal from ([1;0])
    \pgfmathsetmacro{\THETAALPHAONE}{30} % angle of what_alpha from classifier
    \pgfmathsetmacro{\THETAALPHATWO}{60} % angle of what_beta from classifier
    \fill[draw=none, fill=brightnavyblue, opacity=0.4] (-4.5,-4.5) -- (-4.5,4.5) -- (0,4.5) -- (0,0) -- (4.5,0) -- (4.5,-4.5) -- cycle;
    \fill[draw=none, fill=goldenpoppy, opacity=0.4] (0,0) -- (0,4.5) -- (4.5,4.5) -- (4.5,0) -- cycle;
    \draw[help lines, color=gray!30] (-4.9,-4.9) grid (4.9,4.9);
    \draw[rotate=45, fill=gray, opacity=0.4] (0,0) ellipse (2.5 and 1);
    \draw[style=dotted] (0,0) ellipse (2.5 and 2.5);
    \draw[->,thick] (0,0)--({3*cos(\THETACLASS+\THETAALPHAONE)},{3*sin(\THETACLASS+\THETAALPHAONE)}) node[right, xshift=0em, yshift=0em]{\scalebox{1.0}{$w_{\alpha_1}$}};
    \draw[->,thick] (0,0)--({3*cos(\THETACLASS+\THETAALPHATWO)},{3*sin(\THETACLASS+\THETAALPHATWO)}) node[above, xshift=0.2em, yshift=-0.2em]{\scalebox{1.0}{$w_{\alpha_2}$}};
    \draw[<->,thick,color=gray!65] (-5,0)--(5,0) node[right]{};
    \draw[<->,thick,color=gray!65] (0,-5)--(0,5) node[above]{};
    \draw[->,thick,style=dashed] (0,0)--(3,0) node[below right, xshift=-0.3em]{\scalebox{1.0}{$a_1$}};
    \draw[->,thick,style=dashed] (0,0)--(0,3) node[left, xshift=0.5em, yshift=0.5em]{\scalebox{1.0}{$a_2$}};
    \draw[{Latex[length=1mm,width=1mm]}-] (-1.9,1.9) -- (-2.4,2.8) node[above, xshift=-0.6em, yshift=-0.2em]{\scalebox{0.8}{$p(X)$}};
    \draw[{Latex[length=1mm,width=1mm]}-] (0.8,-1.2) -- (1.5,-2.8) node[below right, xshift=-0.5em, yshift=0.5em]{\scalebox{0.8}{$p(\Xhat)$}};
    %\draw[dashed,ultra thick] (0,0)--(0,5) node[right]{};
    %\draw[dashed,ultra thick] (0,0)--(5,0) node[right]{};
    %\draw[color=gray, rotate around={\THETACLASS+\THETABETA:({\x*cos(\THETACLASS+\THETAALPHA)},{\x*sin(\THETACLASS+\THETAALPHA)})}] ({\x*cos(\THETACLASS+\THETAALPHA)},{\x*sin(\THETACLASS+\THETAALPHA)}) ellipse (2.5 and 1.25);
\end{tikzpicture} &

\begin{tikzpicture}[scale=0.25,shading angle from/.style args={line from #1 to #2}{insert path={let \p1=($#2-#1$),\n1={atan2(\y1,\x1)} in},shading angle=\n1}]
    \pgfmathsetmacro{\THETACLASS}{0} % angle of classifier normal from ([1;0])
    \pgfmathsetmacro{\THETAALPHAONE}{15} % angle of what_alpha from classifier
    \pgfmathsetmacro{\THETAALPHATWO}{75} % angle of what_beta from classifier
    \fill[draw=none, fill=brightnavyblue, opacity=0.4] (-4.5,-4.5) -- (-4.5,4.5) -- (0,4.5) -- (0,0) -- (4.5,0) -- (4.5,-4.5) -- cycle;
    \fill[draw=none, fill=goldenpoppy, opacity=0.4] (0,0) -- (0,4.5) -- (4.5,4.5) -- (4.5,0) -- cycle;
    \draw[help lines, color=gray!30] (-4.9,-4.9) grid (4.9,4.9);
    \draw[rotate=45, fill=gray, opacity=0.4] (0,0) ellipse (2.5 and 2);
    \draw[style=dotted] (0,0) ellipse (2.5 and 2.5);
    \draw[<->,thick,color=gray!65] (-5,0)--(5,0) node[right]{};
    \draw[<->,thick,color=gray!65] (0,-5)--(0,5) node[above]{};
    \draw[->,thick] (0,0)--({3*cos(\THETACLASS+\THETAALPHAONE)},{3*sin(\THETACLASS+\THETAALPHAONE)}) node[above right, xshift=-0.3em, yshift=-0.5em]{\scalebox{1.0}{$w_{\alpha_1}$}};
    \draw[->,thick] (0,0)--({3*cos(\THETACLASS+\THETAALPHATWO)},{3*sin(\THETACLASS+\THETAALPHATWO)}) node[above, xshift=0.1em,yshift=-0.2em]{\scalebox{1.0}{$w_{\alpha_2}$}};
    \draw[{Latex[length=1mm,width=1mm]}-{Latex[length=1mm,width=1mm]}] ({2*cos(\THETACLASS+\THETAALPHAONE)},{2*sin(\THETACLASS+\THETAALPHAONE)}) arc (\THETAALPHAONE:\THETAALPHATWO:2);
    \draw[{Latex[length=1mm,width=1mm]}-] (-1.9,1.9) -- (-2.4,2.8) node[above, xshift=-0.2em, yshift=-0.2em]{\scalebox{0.8}{$p(X)$}};
    \draw[{Latex[length=1mm,width=1mm]}-] (1.5,-1.5) -- (2.4,-2.8) node[below right, xshift=-0.5em, yshift=0.5em]{\scalebox{0.8}{$p(\Xhat)$}};
    \draw[->,thick,style=dashed] (0,0)--(3,0) node[below right, xshift=-0.3em]{\scalebox{1.0}{}};
    \draw[->,thick,style=dashed] (0,0)--(0,3) node[left]{\scalebox{1.0}{}};
    \draw[{Latex[length=1mm,width=1mm]}-] (1.55,1.55) -- (2.6,2.6) node[above right, xshift=-0.3em, yshift=-0.3em]{\scalebox{0.7}{$\lambda\uparrow$}};
    %\draw[dashed,ultra thick] (0,0)--(0,5) node[right]{};
    %\draw[dashed,ultra thick] (0,0)--(5,0) node[right]{};
\end{tikzpicture} \\

\hspace{-0.07in}(a) & \hspace{-0.07in}(b) & \qquad & \hspace{-0.07in}(c) & \hspace{-0.07in}(d)

\end{tabular}
    \caption{Explaining simple classifiers in $\R^2$. (a) Visualizing the conditional distribution $p(\Xhat \mid \alpha)$ provides intuition for the linear-Gaussian model. (b) Linear classifier with yellow encoding high probability of $y=1$ (right side), and blue encoding high probability of $y=0$ (left side). Proposition \ref{prop:lingauss/single-hyperplane} shows that the optimal solution to \eqref{eq:objective} is $w^*_\alpha \propto a$ and $w^*_\beta \perp w_\alpha^*$ for $\lambda > 0$. (c-d) For the ``and'' classifier, varying $\lambda$ trades between causal alignment and data representation.}
    \label{fig:lingauss}
\end{figure}
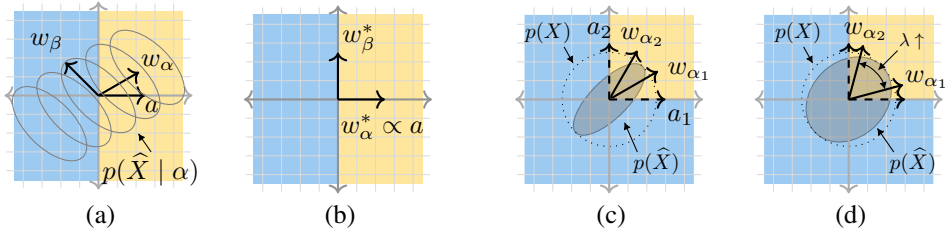

\textbf{Linear classifier.} Consider first a linear separator $p(y = 1 \mid x) = \sigma(a^T x)$, where $a \in \R^N$ denotes the decision boundary normal and $\sigma$ is a sigmoid function (visualized in $\R^2$ in Figure \ref{fig:lingauss}(a)). With a single causal and single noncausal factor ($K=L=1$), learning an explanation consists of finding the $w_{\alpha}, w_{\beta} \in \R^{2}$ that maximize \eqref{eq:objective}. Intuitively, we expect $w_\alpha$ to align with $a$ because this direction allows $\alpha$ to produce the largest change in classifier output statistics. This can be seen by considering the distribution $p(\Xhat \mid \alpha)$ depicted in Figure \ref{fig:lingauss}(a), where we denote $\Xhat=g(\alpha,\beta)$ for convenience. Since the generative model is linear-Gaussian, varying $\alpha$ translates $p(\Xhat \mid \alpha)$ along the direction $w_\alpha$. When this direction is more aligned with the classifier normal $a$, interventions on $\alpha$ cause a larger change in classifier output by moving $p(\Xhat \mid \alpha)$ across the decision boundary. Because the data distribution is isotropic, we expect $\DataFidelity$ to achieve its maximum when $w_{\beta}$ is orthogonal to $w_\alpha$, allowing $w_{\alpha}$ and $w_{\beta}$ to perfectly represent the data distribution. By combining these two insights, we see that the solution of \eqref{eq:objective} is given by $w^*_\alpha \propto a$ and $w^*_\beta \perp w^*_\alpha$ (Figure \ref{fig:lingauss}(b)). 

This intuition is formalized in the following proposition, where for analytical convenience we use the (sigmoidal) normal cumulative distribution function as the classifier nonlinearity $\sigma$:
\begin{proposition}
    \label{prop:lingauss/single-hyperplane}
    Let $\mathcal{X} = \R^N$, $K=1$, $L=N-1$, and $p(Y = 1 \mid x) = \sigma(a^T x)$, where $\sigma$ is the normal cumulative distribution function. Suppose that the columns of $W = [w_\alpha ~ W_\beta]$ are normalized to magnitude $\sqrt{1-\gamma}$ with $\gamma < 1$. Then for any $\lambda > 0$ and for $\DataFidelity(p(\Xhat),p(X)) = -\mathrm{D}_{\mathrm{KL}}(p(X) ~\Vert~ p(\Xhat))$, the objective \eqref{eq:objective} is maximized when $w_{\alpha} \propto a$, $W_{\beta}^T a = 0$, and $W_\beta^T W_\beta=(1-\gamma)I$.
\end{proposition}
%\begin{proposition}
%    \label{prop:lingauss/single-hyperplane}
%    Let $\mathcal{X} = \R^2$, $K = L = 1$, $\norm{w_{\alpha}}_2 = \norm{w_{\beta}}_2 = 1$, and $p(y = 1 \mid x) = \sigma(a^T x)$ where $\sigma$ is the normal cumulative distribution function. Then for any $\lambda > 0$ the objective \eqref{eq:objective} is maximized when $w_{\alpha} \propto a$ and $w_{\beta} \perp a$.
%\end{proposition}
The proof, which is listed in Appendix \ref{sec:supp/proofs/lingauss-single-hyperplane}, follows geometric intuition for the behavior of $\CausalEffect$. This result verifies our objective's ability to construct explanations with our desired properties: the causal factor learns the direction in which the classifier output changes, and the complete set of latent factors represent the data distribution.

\textbf{``And'' classifier.} Now consider the slightly more complex ``and'' classifier parameterized by two orthogonal hyperplane normals $a_1, a_2 \in \R^2$ (Figure \ref{fig:lingauss}(c)) given by $p(Y = 1 \mid x) = \sigma(a_1^T x) \cdot \sigma(a_2^T x)$. This classifier assigns a high probability to $Y = 1$ when both $a_1^T x > 0$ and $a_2^T x > 0$. Here we use $K = 2$ causal factors and $L = 0$ noncausal factors to illustrate the role of $\lambda$ in trading between the terms in our objective. In this setting, learning an explanation entails finding the $w_{\alpha_1}, w_{\alpha_2} \in \R^{2}$ that maximize \eqref{eq:objective}.

Figure \ref{fig:lingauss}(c-d) depicts the effect of $\lambda$ on the learned $w_{\alpha_1}, w_{\alpha_2}$ (see Appendix \ref{sec:supp/lingauss-details} for empirical visualizations). Unlike in the linear classifier case, when explaining the ``and'' classifier there is a tradeoff between the two terms in our objective: the causal influence term encourages both $w_{\alpha_1}$ and $w_{\alpha_2}$ to point towards the upper right-hand quadrant of the data space, the direction that produces the largest variation in class output probability. On the other hand, the isotropy of the data distribution results in the data fidelity term encouraging orthogonality between the factor directions. Therefore, when $\lambda$ is small the causal effect term dominates, aligning the causal factors to the upper right-hand quadrant of the data space (Figure \ref{fig:lingauss}(c)). As $\lambda$ increases (Figure \ref{fig:lingauss}(d)), the larger weight on the data fidelity term encourages orthogonality between the factor directions so that $p(\Xhat)$ more closely approximates $p(X)$. This example illustrates how $\lambda$ must be selected carefully to represent the data distribution while learning meaningful explanatory directions (see Section \ref{sec:methods/training-procedure}).

\section{Experiments with VAE architecture}
\label{sec:vae}
\begin{figure}
    \includegraphics[width=\textwidth]{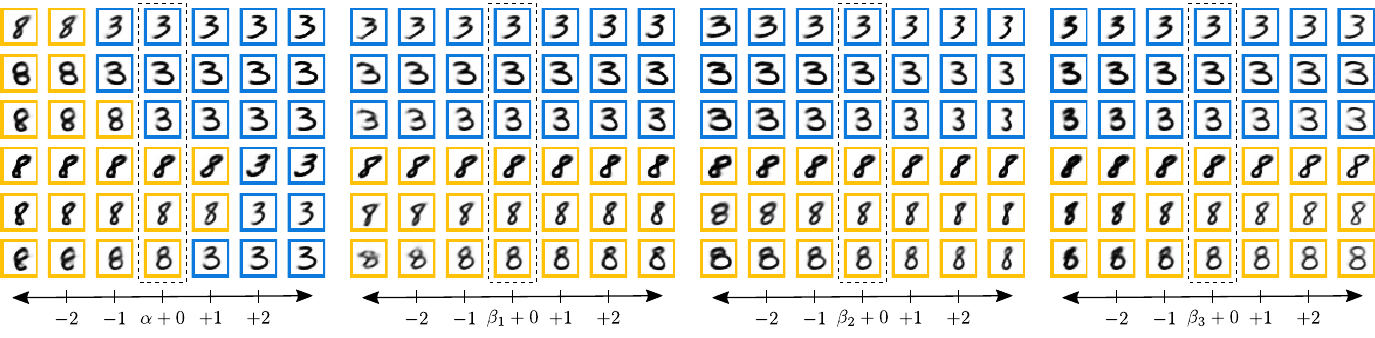} \\
    \hphantom{.} \hspace{0.22in} (a) Sweep $\alpha$ \hspace{0.6in} (b) Sweep $\beta_1$ \hspace{0.6in} (c) Sweep $\beta_2$ \hspace{0.6in} (d) Sweep $\beta_3$
    \caption{Visualizations of learned latent factors. (a) Changing the causal factor $\alpha$ provides the global explanation of the classifier. Images in the center column of each grid are reconstructed samples from the validation set; moving left or right in each row shows $g(\alpha, \beta)$ as a single latent factor is varied. Changing the learned causal factor $\alpha$ affects the classifier output (shown as colored outlines). (b-d) Changing the noncausal factors $\{\beta_i\}$ affects stylistic aspects such as thickness and skew but does not affect the classifier output.}
    \label{fig:mnist-qual}
\end{figure}

In this section we generate explanations of CNN classifiers trained on image recognition tasks, letting $G$ be a set of neural networks and adopting the VAE architecture shown in Figure \ref{fig:graphical-abstract-dag}(a) to learn $g$.

\begin{figure}
    \centering
    \includegraphics{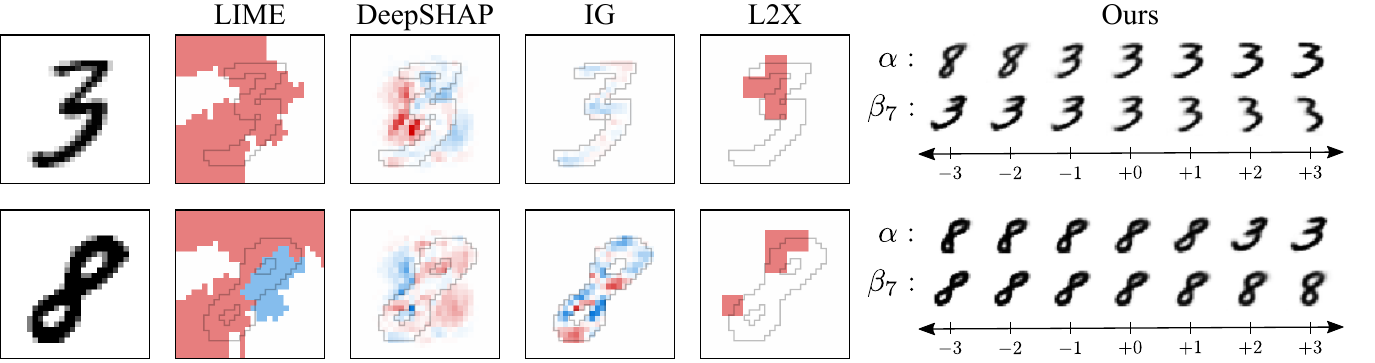}
    \vspace{-0.1in}
    \caption{Compared to popular explanation techniques that generate saliency map-based explanations, our explanations consist of learned aspect(s) of the data, visualized by sweeping the associated latent factors (remaining latent factor sweeps are shown in Appendix \ref{sec:supp/vae-details/comparison}). Our explanations are able to differentiate causal aspects (pixels that define 3 from 8) from purely stylistic aspects (here, rotation).}
    \label{fig:comparison}
\end{figure}

\textbf{Qualitative results.} We train a CNN classifier with two convolutional layers followed by two fully connected layers on MNIST 3 and 8 digits, a common test setting for explanation methods \cite{lundberg2017unified,schwab2019cxplain}. Using the parameter tuning procedure described in Algorithm \ref{alg:parameter-selection}, we select $K = 1$ causal factor, $L = 7$ noncausal factors, and $\lambda = 0.05$. Figure \ref{fig:mnist-qual}(a) shows the global explanation for this classifier and dataset, which visualizes how $g(\alpha,\beta)$ changes as $\alpha$ is modified. We observe that $\alpha$ controls the features that differentiate the digits 3 and 8, so changing $\alpha$ changes the classifier output while preserving stylistic features irrelevant to the classifier such as skew and thickness. By contrast, Figures \ref{fig:mnist-qual}(b-d) show that changing each $\beta_i$ affects stylistic aspects such as thickness and skew but not the classifier output. Details of the experimental setup and training procedure are listed in Appendix \ref{sec:supp/vae-details/mnist} along with additional results.

\textbf{Comparison to other methods.} Figure \ref{fig:comparison} shows the explanations generated by several popular competitors: LIME \cite{ribeiro2016why}, DeepSHAP \cite{lundberg2017unified}, Integrated Gradients (IG) \cite{sundararajan2017axiomatic}, and L2X \cite{chen2018learning}. Each of these methods generates explanations that quantify a notion of relevance of (super)pixels to the classifier output, visualizing the result with a saliency map. While this form of explanation can be appealing for its simplicity, it fails to capture more complex relationships between pixels. For example, saliency map explanations cannot differentiate the ``loops'' that separate the digits 3 and 8 from other stylistic factors such as thickness and rotation present in the same (super)pixels. Our explanations overcome this limitation by instead visualizing latent factors that control different aspects of the data. This is demonstrated on the right of Figure \ref{fig:comparison}, where latent factor sweeps show the difference between classifier-relevant and purely stylistic aspects of the data. Observe that $\alpha$ controls data aspects used by the classifier to differentiate between classes, while the noncausal factor controls rotation. Appendix \ref{sec:supp/vae-details/comparison} visualizes the remaining noncausal factors and details the experimental setup.

\textbf{Quantitative results.} We next learn explanations of a CNN trained to classify t-shirt, dress, and coat images from the Fashion MNIST dataset \cite{xiao2017fashion}. Following the parameter selection procedure of Algorithm \ref{alg:parameter-selection}, we select $K = 2$, $L = 4$, and $\lambda = 0.05$. We evaluate the efficacy of our explanations in this setting using two quantitative metrics. First, we compute the information flow \eqref{eq:information-flow} from each latent factor to the classifier output $Y$. Figure \ref{fig:quantitative}(a) shows that, as desired, the information flow from $\alpha$ to $Y$ is large while the information flow from $\beta$ to $Y$ is small. Second, we evaluate the reduction in classifier accuracy after individual aspects of the data are removed by fixing a single latent factor in each validation data sample to a different random value drawn from the prior $\mathcal{N}(0,1)$. This test is frequently used as a metric for explanation quality; our method has the advantage of allowing us to remove certain data aspects while remaining in-distribution rather than crudely removing features by masking (super)pixels \cite{hooker2019benchmark}. Figure \ref{fig:quantitative}(b) shows this reduction in classifier accuracy. Observe that changing aspects controlled by learned causal factors indeed significantly degrades the classifier accuracy, while removing aspects controlled by noncausal factors has only a negligible impact on the classifier accuracy. Figure \ref{fig:quantitative}(c-d) visualizes the aspects learned by $\alpha_1$ and $\beta_1$. As before, only the aspects of the data controlled by $\alpha$ are relevant to the classifier: changing $\alpha_1$ produces a change in the classifier output, while changing $\beta_1$ affects only aspects that do not modify the classifier output. Appendix \ref{sec:supp/vae-details/fmnist} contains details on the experimental setup and complete results.

\begin{figure}[t]
    \includegraphics[width=\textwidth]{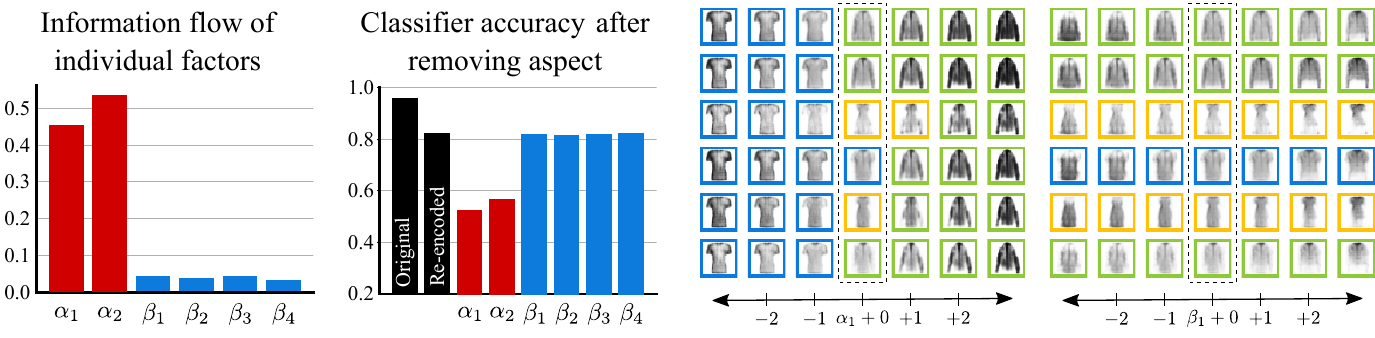} \\
    \hphantom{.}~\hspace{0.54in} (a) \hspace{1.15in} (b) \hspace{0.85in} (c) Sweep $\alpha_1$ \hspace{0.6in} (d) Sweep $\beta_1$
    \caption{(a) Information flow \eqref{eq:information-flow} of each latent factor on the classifier output statistics. (b) Classifier accuracy when data aspects controlled by individual latent factors are removed (original: accuracy on validation set; re-encoded: classifier accuracy on validation set encoded and reconstructed by VAE), showing that learned causal factors (but not noncausal factors) control data aspects relevant to the classifier. (c-d) Modifying $\alpha_1$ changes the classifier output, while modifying $\beta_1$ does not.}
    \label{fig:quantitative}
\end{figure}

\section{Discussion}
\label{sec:discussion}
The central contribution of our paper is a generative framework for learning a rich and flexible vocabulary to explain a black-box classifier, and a method that uses this vocabulary and causal modeling to construct explanations. Our derivation from a causal model allows us to learn explanatory factors that have a causal, not correlational, relationship with the classifier, and the information-theoretic measure of causality that we adapt allows us to completely capture complex causal relationships. Our use of a generative framework to learn independent latent factors that describe different aspects of the data allows us to ensure that our explanations respect the data distribution.

Applying this framework to practical explanation tasks requires selecting a generative model architecture, and then training this generative model using data relevant to the classification task. The data used to train the explainer may be the original training set of the classifier, but more generally it can be any dataset; the resulting explanation will reveal the aspects in that specific dataset that are relevant to the classifier. The user must also select a generative model $g$ with appropriate capacity. Underestimating this capacity could reduce the effectiveness of the resulting explanations, while overestimating this capacity will needlessly increase the training cost. We explore this selection further in Appendix \ref{sec:supp/vae-capacity} both empirically and by using results from \cite{feder1994relations} to show how the value of $I(\alpha;Y)$ can be interpreted as a ``certificate'' of sufficient generative model capacity.

Our framework combining generative and causal modeling is quite general. Although we focused on the use of learned data aspects to generate explanations by visualizing the effect of modifying learned causal factors, the learned representation could also be used to generate counterfactual explanations --- minimal perturbations of a data sample that change the classifier output \cite{wachter2018counterfactual,miller2019explanation}. Our framework would address two common challenges in counterfactual explanation: because we can optimize over a low-dimensional set of latent factors, we avoid a computationally infeasible search in input space, and because each point in space maps to an in-distribution data sample, our model naturally ensures that perturbations result in a valid data point. Another promising avenue for future work is relaxing the independence structure of learned causal factors. Although this would result in a more complex expression for information flow, the sampling procedure we use to compute causal effect would generalize naturally; the more challenging obstacle would be learning latent factors with nontrivial causal structure. Finally, techniques that make the classifier-relevant latent factors more interpretable or better communicate the aspects controlled by each latent factor to humans would improve the quality of our generated explanations.

\section*{Broader impacts}
\label{sec:broader-impacts}
Explanation methods have the potential to play a major role in enabling the safe and fair deployment of machine learning systems \cite{kroll2017accountable,danks2017regulating}, and explainability is a oft-mentioned constraint in their legal and ethical analysis. Policy discussions about machine learning have increasingly turned to principles of transparency and fairness \cite{karsten2020new}, with some legal scholars arguing that the 2016 European General Data Protection Regulation (GDPR) contains a ``right to explanation'' \cite{malgieri2017why}, and recent G20 and OECD recommendations both identifying ``transparency and explainability'' as important principles for the development of machine learning algorithms \cite{g202019g20,oecd2020recommendation}.

The growing literature on explainability that our work contributes to has the potential to improve the transparency and fairness of machine learning systems and increase the level of trust users place in their decisions. Yet these explanation methods, often built from complex and nontransparent components and each proposing subtly different notions of explanation, also risk providing deceptively incomplete understanding of systems used in sensitive applications, or providing false assurances of fairness and lack of bias (see, e.g., \cite{rudin2019stop}). This criticism may be especially true for our method, which constructs explanations using neural networks that are themselves difficult to understand. For the explanation literature to have a positive impact, it is necessary for explanations to be easily yet precisely understood by the nontechnical generalists deploying and regulating machine learning systems. We believe that causal perspective used in this work is valuable in this regard because causality has been identified as a vocabulary appropriate for translating technical concepts to psychological \cite{miller2019explanation} and legal frameworks \cite{kroll2017accountable,wachter2018counterfactual}. We also believe our analysis with simple models is important because it endows our explanations with some theoretical grounding. However, a critical need remains for more interdisciplinary research examining how end users understand the outputs of explanation tools (e.g., \cite{tonekaboni2019what}) and how technical tools can be brought to bear to address identified deficiencies.

\begin{ack}
This work was supported by NSF grant CCF-1350954, a gift from the Alfred P.\ Sloan Foundation, and the National Defense Science \& Engineering Graduate (NDSEG) Fellowship.
\end{ack}

%\small
\bibliography{moshaughnessy.bib,refs.bib}

\newpage
\appendix

\section{Intuition for and variants of causal influence metric}
\label{sec:supp/causal-obj-variants}
\textbf{Intuition for causal influence objective.} To better understand the causal portion of our objective \eqref{eq:causal-effect-decomposition}, we use standard identities to decompose it as
\begin{equation}
    \mathcal{C} = I(Y; \alpha) = H(Y) - \E_{\alpha}[H(Y \mid \alpha)],
    \label{eq:causal-effect-decomposition}
\end{equation}
where
\begin{equation}
    p(y \mid \alpha) = \int_\beta \int_x p(y \mid x) p(x \mid \alpha, \beta) p(\beta) dx d\beta.
    \label{eq:pygivenalpha}
\end{equation}

The conditional distribution \eqref{eq:pygivenalpha} can be interpreted as the probability of $Y=y$ for a fixed value of $\alpha$, averaged over the values of $\beta$. The decomposition in \eqref{eq:causal-effect-decomposition} therefore shows that $\CausalEffect$ is the \emph{reduction in uncertainty about $Y$ provided by knowledge of $\alpha$}, where this reduction is measured in a global sense in that the effect of $\beta$ is averaged together to produce a single probability estimate for $Y$ and fixed $\alpha$.

As an example, consider the color classifier and generative mapping shown in Figure \ref{fig:graphical-abstract-dag}(a), in which $f$ classifies based on color. The first term in \eqref{eq:causal-effect-decomposition} represents how similar the classifier output is for all objects in the training set. The second term represents how similar the classifier output for groups of objects is, on average, after being grouped by $\alpha$. A large $\CausalEffect = I(\alpha ; Y)$ means that grouping by $\alpha$ significantly increases the confidence the classifier has that objects in each group are of the same class. In this case, grouping by $\alpha = \text{`color'}$ has a much larger effect on the classifier output --- and therefore results in a larger $\CausalEffect$ --- than grouping by $\alpha = \text{`shape'}$ would, since grouping the objects by color results in the classifier gaining much more confidence that each group shares the same class.

\textbf{Variants of causal objective.} Consider the following variants of the \emph{joint, unconditional} objective $\CausalEffect = I(\alpha ; Y)$, our measure of causal influence from Section \ref{sec:methods/causal-influence}:
\begin{enumerate}
    \item \emph{Independent, unconditional:} $\CausalEffect_{iu} = \frac1K \sum_i I(\alpha_i; Y)$
    \item \emph{Independent, conditional:} $\CausalEffect_{ic} = \frac1K \sum_i I(\alpha_i ; Y \mid \alpha_{\neg i}, \beta)$, where $\alpha_{\neg i} = \{\alpha_j\}_{j \neq i}$
    %\item \emph{Joint, unconditional:} $\CausalEffect_{ju} = I(\alpha ; Y)$ (i.e., the objective $\CausalEffect$ used in the main text)
    \item \emph{Joint, conditional:} $\CausalEffect_{jc} = I(\alpha; Y \mid \beta)$
\end{enumerate}
Each objective variant gives rise to a classifier explanation that has a causal interpretation, but as we will show, the \emph{character} of each is subtly different. The following proposition begins to explore these differences by relating them using information-theoretic quantities.
\begin{proposition}[Relationship between candidate causal objectives]
    \label{prop:information-flow-relationships}
    The following hold in the DAG of Figure \ref{fig:graphical-abstract-dag}(b):
    \begin{enumerate}
        \item[\emph{(a)}] $\CausalEffect = \CausalEffect_{iu} + \frac1K \sum_{i=1}^K I(\alpha_{\neg i} ; Y \mid \alpha_i)$.
        \item[\emph{(b)}] $\CausalEffect_{jc} = \CausalEffect_{ic} + \frac1K \sum_{i=1}^K I(\alpha_{\neg i} ; Y \mid \beta)$.
        \item[\emph{(c)}] $\CausalEffect_{jc} = \CausalEffect + I(\alpha; \beta \mid Y)$.
        \item[\emph{(d)}] $\CausalEffect_{ic} = \CausalEffect_{iu} + \frac1K \sum_i I(\alpha_i ; \alpha_{\neg i}, \beta \mid Y)$.
    \end{enumerate}
\end{proposition}

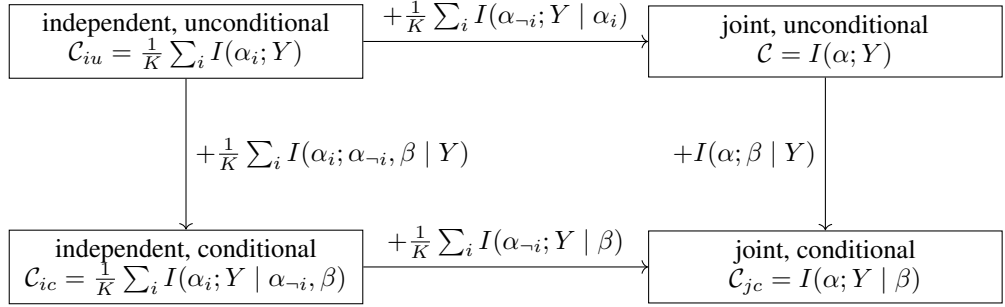
\begin{figure}
    \centering
    \begin{tikzpicture}[every text node part/.style={align=center}]
    \node[draw=black, minimum width=4.7cm] (indepuncond) {independent, unconditional\\$\CausalEffect_{iu} = \frac1K \sum_i I(\alpha_i ; Y)$};
    \node[right of=indepuncond, draw=black, minimum width=4.7cm, xshift=7.5cm] (jointuncond) {joint, unconditional\\$\CausalEffect = I(\alpha ; Y)$};
    \node[below of=jointuncond, draw=black, minimum width=4.7cm, yshift=-2cm] (jointcond) {joint, conditional\\$\CausalEffect_{jc} = I(\alpha ; Y \mid \beta)$};
    \node[below of=indepuncond, draw=black, minimum width=4.7cm, yshift=-2cm] (indepcond) {independent, conditional\\$\CausalEffect_{ic} = \frac1K\sum_i I(\alpha_i ; Y \mid \alpha_{\neg i}, \beta)$};
    \path[->] (indepuncond) edge node [above] {$+\frac1K \sum_i I(\alpha_{\neg i} ; Y \mid \alpha_i)$} (jointuncond);
    \path[->] (indepcond) edge node [above] {$+\frac1K \sum_i I(\alpha_{\neg i} ; Y \mid \beta)$} (jointcond);
    \path[->] (jointuncond) edge node [left] {$+I(\alpha; \beta \mid Y)$} (jointcond);
    \path[->] (indepuncond) edge node [right] {$+ \frac1K \sum_i I(\alpha_i;\alpha_{\neg i}, \beta \mid Y)$} (indepcond);
    \end{tikzpicture}
    \caption{Graphical representation of relationships between causal objective variants derived from Proposition \ref{prop:information-flow-relationships}.}
    \label{fig:information-flow-relationships}
\end{figure}

These relationships are depicted visually in Figure \ref{fig:information-flow-relationships} and proved in Appendix \ref{sec:supp/proofs/information-flow-relationships}. Note that only (c) and (d) use the independence of the latent variables in our DAG. The ``adjustment factors'' that relate the objective variants can be interpreted as follows:
\begin{enumerate}
    \item By conditioning on other latent factors (i.e., using $\CausalEffect_{ic}$, $\CausalEffect_{iu}$, or $\CausalEffect_{jc}$ rather than $\CausalEffect$) we include the ``adjustment factor'' $\frac1K \sum_i I(\alpha_i ; \alpha_{\neg i}, \beta \mid Y)$ (in the ``independent'' case) or $I(\alpha ; \beta \mid Y)$ (in the ``joint'' case) in the objective. These terms encourage complex interactions between latent factors within each group of similarly-classified points. On the one hand, the stastistical pattern that these terms encourage arises naturally from the DAG in Figure \ref{fig:graphical-abstract-dag}(b): although the latent factors are independent, conditioning on $Y$ renders them dependent. This conditional dependence pattern is often referred to as Berkson's paradox or the ``explaining away'' phenomenon. To illustrate this concept, consider a classifier that classifies paintings at an auction as $Y \in \{\text{`sold'},~\text{`not sold'}\}$ based on the learned latent factors $z_1=\text{`beautiful'}$ and $z_2=\text{`historical value'}$, which we assume to be independent. Once $Y$ is known, however, $z_1$ and $z_2$ are rendered dependent: learning that a sold painting does not have historical value would allow us to infer that it is likely to be beautiful. On the other hand, we do not in general expect that our learned latent factors, which we encourage to be independent, will correspond to semantically meaningful features, so we may not expect them to fit this ``explaining away'' conditional dependence pattern.
    \item By jointly considering the causal factors $\alpha$ rather than summing the causal influence of each $\alpha_i$ (i.e., by using $\CausalEffect$ rather than $\CausalEffect_{iu}$, or $\CausalEffect_{jc}$ rather than $\CausalEffect_{ic}$) we include the ``adjustment factor'' $\frac1K \sum_{i=1}^K I(\alpha_{\neg i}; Y \mid \alpha_i)$ in the objective. This term encourages each learned causal factor to make the remaining causal factors more predictable given the classifier output $Y$, encouraging \emph{interactions} between latent factors to have an effect on the classifier output probability. We consider this to be positive, but using an independent objective might aid in visualizing the relationship between the latent space and data space.
\end{enumerate}

The next section provides more intuition for these objectives in the context of the linear-Gaussian generative map and simple classifiers introduced in Section \ref{sec:lingauss}.
\clearpage

\section{Detailed analysis with linear-Gaussian generative map}
\label{sec:supp/lingauss-details}
\begin{figure}
    \input{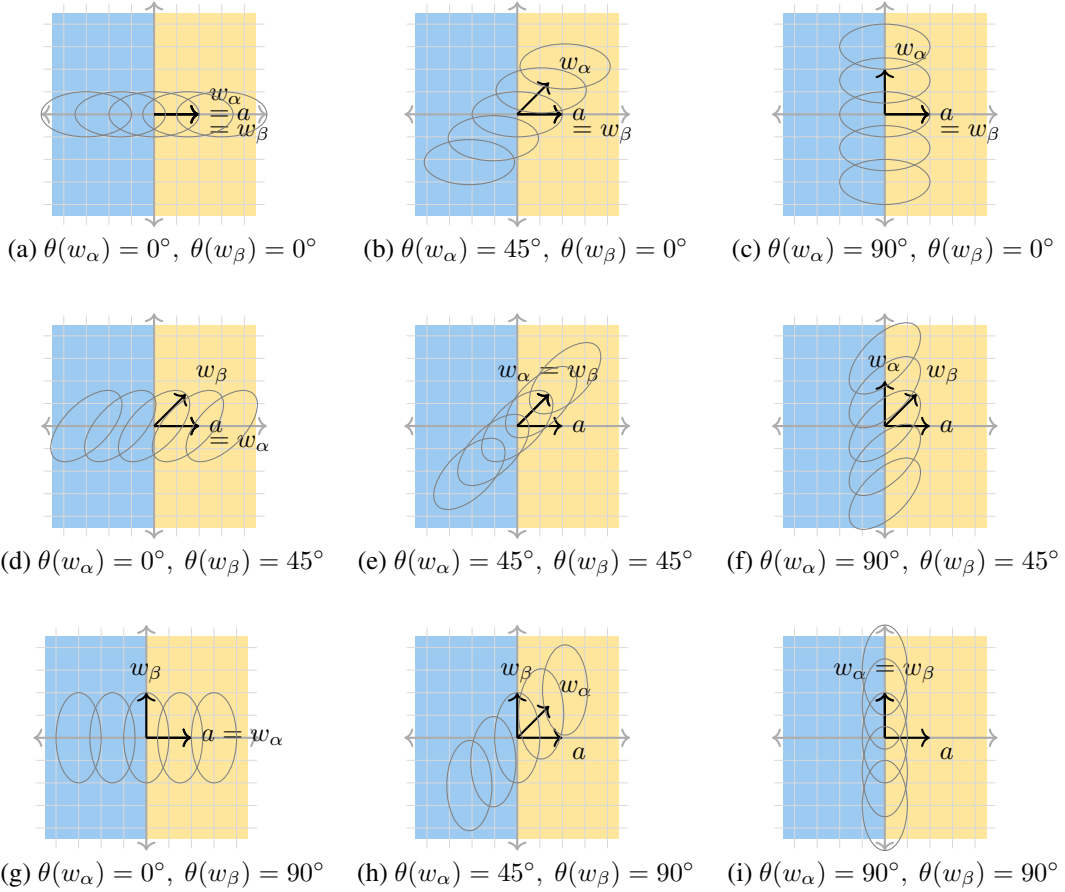}
    \caption{Distributions $p(x \mid \alpha)$ for the linear-Gaussian generative map and single hyperplane classifier when $a = [1,~0]^T$. The orientation of $w_\alpha$ controls the direction in which the probability mass of $p(x \mid \alpha)$ shifts as $\alpha$ is varied, while the orientation of $w_\beta$ controls the rotation of each distribution $p(x \mid \alpha)$.}
    \label{fig:lingauss-single-hyperplane}
\end{figure}

In this section we provide empirical simulations supporting the analysis with a linear-Gaussian generative map in Section \ref{sec:lingauss}. Recall that we use the isotropic data distribution $X \sim \mathcal{N}(0,I)$, latent space prior $(\alpha,\beta) \sim \mathcal{N}(0,I)$, and
\begin{equation*}
    g(\alpha,\beta) = \begin{bmatrix} W_\alpha & W_\beta \end{bmatrix} \begin{bmatrix} \alpha \\ \beta \end{bmatrix} + \varepsilon,
\end{equation*}
where $W_\alpha \in \R^{N \times K}$, $W_\beta \in \R^{N \times L}$, and $\varepsilon \sim \mathcal{N}(0,\gamma I)$.

\textbf{Linear classifier.} Consider first the linear separator in $\R^2$ from Section \ref{sec:lingauss}, $p(Y = 1 \mid x) = \sigma(a^T x)$. With $K = L = 1$, learning an explanation entails learning the $w_\alpha, w_\beta \in \R^2$ that maximize the objective \eqref{eq:objective}. As shown in Proposition \ref{prop:lingauss/single-hyperplane}, the data representation term $\DataFidelity$ encourages $w_\alpha \perp w_\beta$; here we focus on the causal influence term $\CausalEffect$. The decomposition in \eqref{eq:causal-effect-decomposition} shows that $\CausalEffect$ depends on both $p(Y)$ and $p(Y \mid \alpha)$; Figure \ref{fig:lingauss-single-hyperplane} visualizes how the distributions $p(Y \mid \alpha)$ change with $\alpha$ (gray ellipses) and $w_\alpha, w_\beta$ (subplots). Note first that the isotropy of $p(\alpha)$ means that $p(Y)$ has equal probability mass on either side of the classifier decision boundary, regardless of $w_\alpha$ and $w_\beta$. This implies that $H(Y)$ is invariant to $w_\alpha$ and $w_\beta$ for this classifier, a fact formalized in the proof of Proposition \ref{prop:lingauss/single-hyperplane}.

We next explore the role of $w_\alpha$ and $w_\beta$ in $p(x \mid \alpha)$ (and therefore $p(y \mid \alpha)$). Our causal objective $\CausalEffect$ is large when the $p(y \mid \alpha)$ have low entropy in expectation over $\alpha$. Note from Figure \ref{fig:lingauss-single-hyperplane} that $w_\alpha$ controls the direction in which the probability mass of $p(x \mid \alpha)$ shifts as $\alpha$ is varied, while $w_\beta$ controls the rotation of each distribution $p(x \mid \alpha)$. The causal objective $\CausalEffect$ is maximized when the entropy of $p(y \mid \alpha)$ (in expectation over $\alpha$) is smallest --- in other words, when the distributions $p(x \mid \alpha)$ have as little overlap possible with the classifier decision boundary. From Figure \ref{fig:lingauss-single-hyperplane}, we observe that this occurs when $w_\alpha$ is aligned with the decision boundary normal ($w_\alpha \propto a$) and when $w_\beta$ is orthogonal to the decision boundary normal ($w_\beta \perp a$). This selection of $w_\alpha$ and $w_\beta$ minimizes the range of $\alpha$ for which $p(x \mid \alpha)$ contains mass on both sides of the decision boundary.

\begin{figure}
    \centering
    \includegraphics[width=\textwidth]{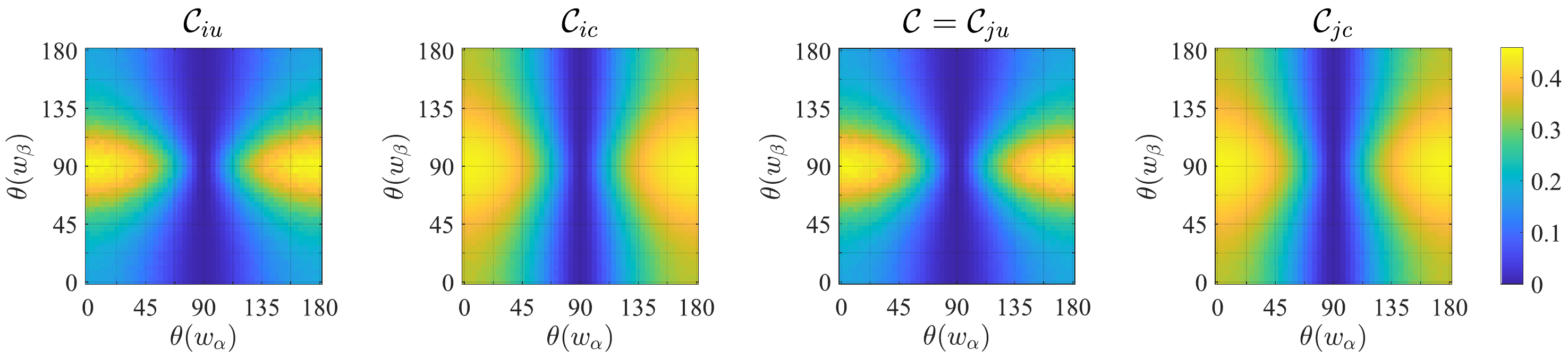}
    \caption{Value of each causal objective variant in the linear-Gaussian generative map, linear classifier setting described in Section \ref{sec:lingauss}, as the orientations of $w_\alpha$ and $w_\beta$ are varied. The classifier decision boundary normal is $\theta(a) = 0^\circ$. Each variant is maximized when $w_\alpha \propto a$ (i.e., $\theta(w_\alpha) = 0^\circ$) and $w_\beta \perp a$ (i.e., $\theta(w_\beta) = 90^\circ$). $\CausalEffect = \CausalEffect_{ju}$ refers to the causal objective \eqref{eq:causal-effect} used in the main text.}
    \label{fig:lingauss-objective-singlehyperplane}
\end{figure}

Figure \ref{fig:lingauss-objective-singlehyperplane} shows the value of each of the causal objective variants described in Appendix \ref{sec:supp/causal-obj-variants} as the orientation of $w_\alpha$ and $w_\beta$ with respect to the classifier decision boundary normal $a$ are varied. For each combination of angles, we compute the causal objective using the sample-based estimate described in Appendix \ref{sec:supp/sampling} with $N_{\alpha} = 2500$, $N_{\beta} = 500$, and the logistic sigmoid function $\sigma$ with steepness $5$. (Note that in the training procedure we achieve satisfactory results with much lower $N_\alpha, N_\beta$.) These results verify the intuition presented above and formalized in Proposition \ref{prop:lingauss/single-hyperplane}: the causal effect is greatest when $w_\alpha \propto a$ and $w_\beta \perp a$. As noted in Section \ref{sec:lingauss}, in this setting both $\CausalEffect$ and $\DataFidelity$ encourage $w_\alpha$ and $w_\beta$ to be orthogonal.

\begin{figure}
    \centering
    \includegraphics[width=\textwidth]{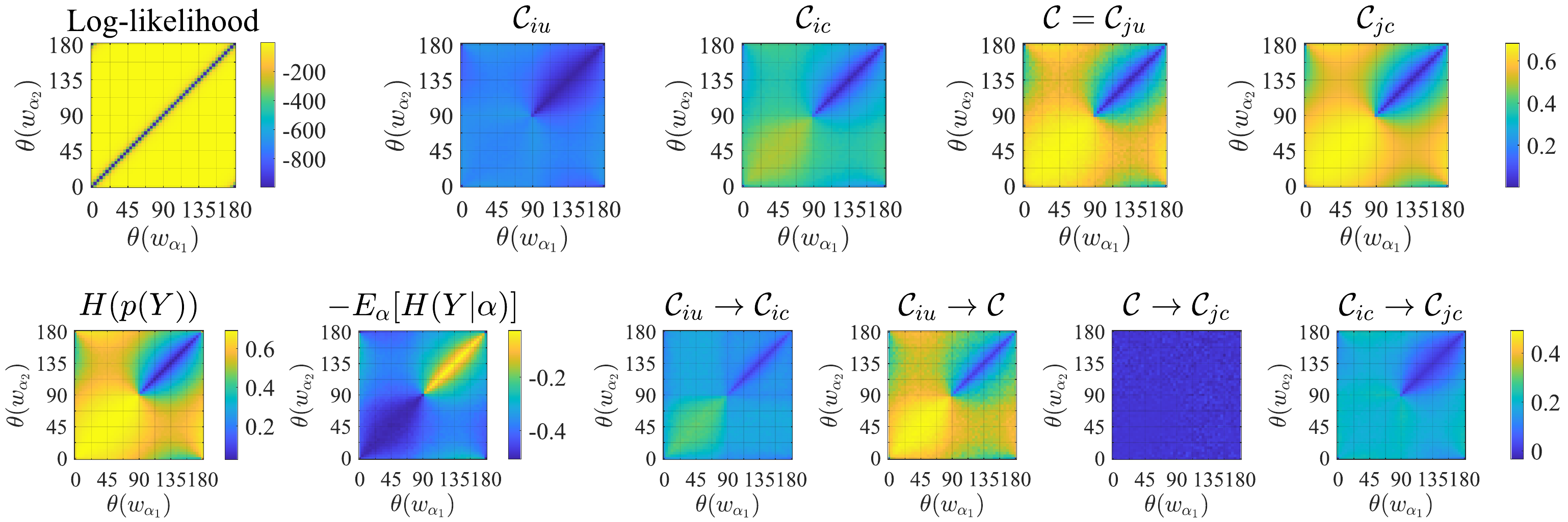}
    \caption{Empirically-computed values of terms relevant to the causal objective variants in the linear-Gaussian generative map, ``and'' classifier setting described in Section \ref{sec:lingauss}. The angles of the classifier decision boundary normals are $\theta(a_1) = 0^\circ$ and $\theta(a_2) = 90^\circ$. Top row: log-likelihood used as $\DataFidelity$; causal objective variants from Appendix \ref{sec:supp/causal-obj-variants}. $\CausalEffect = \CausalEffect_{ju}$ refers to the causal objective \eqref{eq:causal-effect}. Bottom row: terms in decomposition \eqref{eq:causal-effect-decomposition}; ``adjustment factors'' from Proposition \ref{prop:information-flow-relationships}.}
    \label{fig:lingauss-objective-and}
\end{figure}

\begin{figure}
    \centering
    \includegraphics[width=\textwidth]{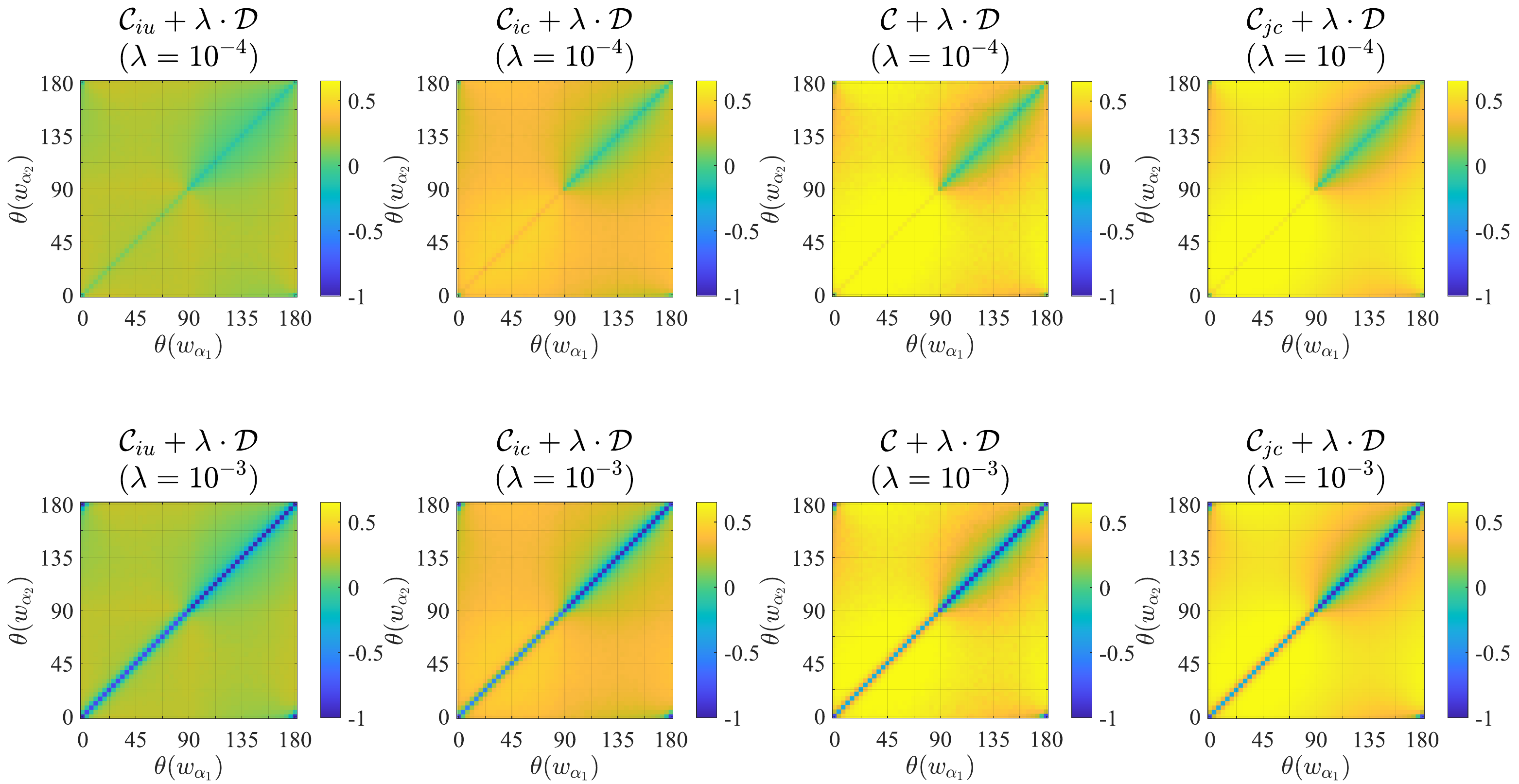}
    \caption{Empirically-computed value of combined objective \eqref{eq:objective} for the causal objective variants in the linear-Gaussian generative map, ``and'' classifier setting described in Section \ref{sec:lingauss}. The angles of the classifier decision boundary normals are $\theta(a_1) = 0^\circ$ and $\theta(a_2) = 90^\circ$. As $\lambda$ increases, the increased weight of the data representation term in the objective encourages the learned $w_{\alpha_1}$ and $w_{\alpha_2}$ to be more orthogonal to better represent the isotropic distribution of the data.}
    \label{fig:lingauss-combined-objective-and}
\end{figure}

\textbf{``And'' classifier.} We now consider the ``and'' classifier in $\R^2$ from Section \ref{sec:lingauss}, $p(Y = 1 \mid x) = \sigma(a_1^T x) \cdot \sigma(a_2^T x)$, where we learn $K = 2$ causal explanatory factors and $L = 0$ noncausal factors. In this setting learning an explanation consists of learning $w_{\alpha_1}, w_{\alpha_2} \in \R^2$ maximizing \eqref{eq:objective}.

Figure \ref{fig:lingauss-objective-and} shows how the value of the causal objective changes with the learned generative mapping in the linear-Gaussian setting of Section \ref{sec:lingauss}. The top row shows the terms in the objective \eqref{eq:objective}: the likelihood and the causal objective variants described in Appendix \ref{sec:supp/causal-obj-variants}. The bottom row shows the components of these causal objective variants, which provide further intuition for their differences: the first two plots show the decomposition of $\CausalEffect = \CausalEffect_{ju}$ from \eqref{eq:causal-effect-decomposition}, and the remaining plots show the ``adjustment factors'' from Proposition \ref{prop:information-flow-relationships} and Figure \ref{fig:information-flow-relationships} that describe the differences between the causal influence objective variants. The logistic sigmoid with steepness 100 is used to implement the classifier, and the causal influence objective variants are computed with $N_\alpha = 2500$ and $N_\beta = 500$.

With the exception of the variant $\CausalEffect_{iu}$, each of these causal objectives is maximized when $w_{\alpha_1}$ and $w_{\alpha_2}$ are aligned in the direction of maximum classifier change: $\theta(w_{\alpha_1}) = \theta(w_{\alpha_2})$ when $a_1 = [1,~ 0]^T$ and $a_2 = [0,~ 1]^T$ as in our example (see Figure \ref{fig:lingauss}(c-d)). Because with this classifier $\CausalEffect$ does not encourage $w_{\alpha_1} \perp w_{\alpha_2}$, here the data representation term $\DataFidelity$ serves to regularize $\CausalEffect$. Figure \ref{fig:lingauss-combined-objective-and} shows the value of the combined objective \eqref{eq:objective} for each causal influence variant and two different values of $\lambda$. We observe that as $\lambda$ increases and the weight of the data representation term increases, the optimal angles of $w_{\alpha_1}$ and $w_{\alpha_2}$ move in opposing directions from $45^\circ$ (the angle of normal bisecting $a_1$ and $a_2$). This supports the intuition described in Section \ref{sec:lingauss} and stylized in Figure \ref{fig:lingauss}(c-d).
\clearpage

\section{Proofs}
\label{sec:supp/proofs}
\subsection{Proof of Proposition \ref{prop:information-flow}}
\label{sec:supp/proofs/information-flow}

Proposition \ref{prop:information-flow} states that information flow coincides with mutual information in our DAG. Here we prove a generalization of the proposition that is also helpful when considering the conditional causal influence objective variants in Appendix \ref{sec:supp/causal-obj-variants}. Specifically, we consider the information flow from $U$ to $V$ \emph{imposing $W$}:
\begin{definition}[Ay and Polani 2008 \cite{ay2008information}]
    Let $U$, $V$, and $W$ be disjoint subsets of nodes. The \emph{information flow from $U$ to $V$ imposing $W$}, denoted $I(U \to V \mid W)$, is 
    \begin{equation*}
        \E_{w \sim W} \left[ \int_U p(u \mid do(w)) \int_V p(v \mid do(u), do(w)) \log \frac{p(v \mid do(u), do(w))}{\int_{u'} p(u' \mid do(w)) p(v \mid do(u'), do(w))} dV dU \right],
    \end{equation*}
    where $do(w)$ represents an intervention in a model that fixes $w$ to a specified value regardless of the values of its parents \cite{pearl2009causality}.
\end{definition}

\begin{proposition}[Information flow in our DAG]
    \label{prop:information-flow-imposing}
    The information flow from $\alpha$ to $Y$ imposing $\beta$ in the DAG of Figure \ref{fig:graphical-abstract-dag}(b) coincides with the mutual information of $\alpha$ and $Y$ conditioned on $\beta$,
    \begin{equation*}
        I(\alpha \to Y \mid do(\beta)) = I(\alpha; Y \mid \beta),
    \end{equation*}
    where conditional mutual information is defined as $I(X ; Y \mid Z) = \E_{X,Y,Z} \left[ \log \frac{p(x,y \mid z}{p(x \mid z) p(y \mid z)} \right]$.
\end{proposition}

\begin{proof}
    The proof follows from the ``action/observation exchange'' rule of the $do$-calculus \cite[Thm.~3.4.1]{pearl2009causality}. This rule asserts that $p(y \mid do(x), do(z), w) = p(y \mid do(x), z, w)$ if $Y \perp Z \mid X, W$ in $\mathcal{G}_{\overline{X}\underline{Z}}$, the causal model modified to remove connections entering $X$ and leaving $Z$. When applied to our model, it yields
    \begin{enumerate}
        \item $p(Y \mid do(\alpha)) = p(Y \mid \alpha)$ (because $Y \perp \alpha$ in $\mathcal{G}_{\underline{\alpha}}$);
        \item $p(\alpha \mid do(\beta)) = p(\alpha \mid \beta)$ (because $\alpha \perp \beta$ in $\mathcal{G}_{\underline{\beta}}$); and
        \item $p(Y \mid do(\alpha), do(\beta)) = p(Y \mid \alpha, \beta)$ (because $Y \perp (\alpha, \beta)$ in $\mathcal{G}_{\underline{\alpha},\underline{\beta}}$).
    \end{enumerate}
    %
    %(a) Starting with the definition of the information flow from $\alpha$ to $Y$,
    %\begin{align*}
    %    I(\alpha \to Y) &= \int_{\alpha} p(\alpha) \int_{Y} p(Y \mid do(\alpha)) \log \frac{p(Y \mid do(\alpha))}{\int_{\alpha'} p(\alpha') p(Y \mid do(\alpha'))} \\
    %    &= \int_{\alpha} p(\alpha) \int_{Y} p(Y \mid \alpha) \log \frac{p(Y \mid \alpha)}{\int_{\alpha'} p(\alpha') p(Y \mid \alpha')} \\
    %    &= \int_{\alpha, Y} p(\alpha, Y) \log \frac{p(Y \mid \alpha)}{p(Y)} \\
    %    &= \int_{\alpha, Y} p(\alpha, Y) \log \frac{p(Y, \alpha)}{p(Y) p(\alpha)} \\
    %    &= I(\alpha; Y).
    %\end{align*}
    
    %(b)
    Starting with the definition of the information flow from $\alpha$ to $Y$ imposing $\beta$, we have that
    \begin{align*}
        I(\alpha \to Y \mid do(\beta)) &= \E_{\beta} \left[ \int_{\alpha} p(\alpha \mid do(\beta)) \int_{Y} p(Y \mid do(\alpha), do(\beta)) \right.\\
        &\qquad\qquad \times \left. \log \frac{p(Y \mid do(\alpha), do(\beta))}{\int_{a'} p(\alpha=a' \mid do(\beta)) p(Y \mid do(\alpha=a'), do(\beta))} \right] dY d\alpha \\
        &= \E_{\beta} \left[ \int_{\alpha} p(\alpha \mid \beta) \int_{Y} p(Y \mid \alpha, \beta) \right.\\
        &\qquad\qquad \times \left. \log \frac{p(Y \mid \alpha, \beta)}{\int_{a'} p(\alpha=a' \mid \beta) p(Y \mid \alpha=a', \beta)} \right] dY d\alpha \\
        &= \E_{\beta} \left[ \int_{\alpha,Y} p(Y, \alpha \mid \beta) \log \frac{p(Y \mid \alpha, \beta)}{p(Y \mid \beta)} \right] dY d\alpha \\
        &= \int_{\beta} p(\beta) \int_{\alpha,Y} p(Y, \alpha \mid \beta) \log \frac{p(Y \mid \alpha, \beta)}{p(Y \mid \beta)} dY d\alpha d\beta \\
        &= \int_{\beta} p(\beta) \int_{\alpha,Y} p(Y, \alpha \mid \beta) \log \frac{p(Y \mid \alpha, \beta) p(\alpha \mid \beta)}{p(Y \mid \beta) p(\alpha \mid \beta)} dY d\alpha d\beta \\
        &= \int_{\beta} p(\beta) \int_{\alpha,Y} p(Y, \alpha \mid \beta) \log \frac{p(Y, \alpha \mid \beta)}{p(Y \mid \beta) p(\alpha \mid \beta)} dY d\alpha d\beta \\
        &= I(\alpha; Y \mid \beta).
    \end{align*}
\end{proof}

Proposition \ref{prop:information-flow} follows from Proposition \ref{prop:information-flow-imposing} by imposing the null set.

\subsection{Proof of Proposition \ref{prop:lingauss/single-hyperplane}}
\label{sec:supp/proofs/lingauss-single-hyperplane}

With $K = 1$ we can decompose $\CausalEffect$ as
\begin{equation}
    \label{eq:causal-effect-decomposition-proof}
    \CausalEffect = I(Y ; \alpha) = H(Y) - H(Y \mid \alpha).
\end{equation}
where $H$ denotes entropy of a discrete random variable \cite{cover2006elements}. First consider the entropy term $H(Y)$. From the illustrations of $p(\Xhat \mid \alpha)$ in Figure \ref{fig:lingauss-single-hyperplane}, we can see in $\R^2$ that this entropy is constant for all values of $w_\alpha$ and $w_\beta$: regardless of their angle and offsets, the aggregate set of distributions $p(\Xhat \mid \alpha)$ is symmetric about the origin and so the probability mass of $p(\Xhat)$ is spread symmetrically across both sides of the decision boundary. This idea is generalized in the following lemma, which shows that $H(Y)$ is equal to $\log(2) \approx 0.69$ nats for all values of $W$:
\begin{lemma}
    \label{lem:ent}
    Under the conditions of Propsition \ref{prop:lingauss/single-hyperplane}, $H(Y)=\log(2)$ nats for all $W \in \R^{N \times N}$.
\end{lemma}

\begin{proof}
    Since $(\alpha,\beta) \sim \mathcal{N}(0,I)$, we have
    $\Xhat \sim \mathcal{N}(0,WW^T + \gamma I)$. Letting $U = a^T X$, we have $U\sim \mathcal{N}(0,a^T(WW^T + \gamma I)a)$ which we note has an even probability density function. Considering the classifier output probability marginalized over the generated inputs $\Xhat$, we have
    \begin{align*}
        p(Y=1) &= \E_{\Xhat}[p(Y=1 \mid \Xhat)]
        \\ &= \E_{\Xhat}[\sigma(a^T \Xhat)]
        \\ &= \E_U [\sigma(U)]
        \\ &= \E_U[\sigma(U)-0.5] + 0.5
        \\ &\overset{(\star)}{=} 0.5
    \end{align*}
    where in ($\star$) we use the fact that since $U$ has an even probability density and $\sigma(U)-0.5$ is an odd function, we have that $\E_U[\sigma(U)-0.5] = 0$. Letting $h_b(p)=-(p \log p + (1-p) \log (1-p))$ denote the binary entropy function, we have that $H(\widehat{Y}) = h_b(p(\widehat{Y}=1))=h_b(0.5)=\log(2)$ nats.
\end{proof}

We now consider the second term in \eqref{eq:causal-effect-decomposition-proof}, the conditional entropy $H (Y \mid \alpha)$. In $\R^2$ (Figure \ref{fig:lingauss-single-hyperplane}), this term corresponds to the average over $\alpha$ of the classification entropies for each distribution $p(\Xhat \mid \alpha)$ (depicted as individual ellipses). Intuitively, this entropy is small when many of the conditional distributions $p(\Xhat \mid \alpha)$ lie almost entirely on a single side of the decision boundary (corresponding to high classifier output agreement within each distribution, and therefore low entropy). The orientation of $w_\beta$ can reduce this term by rotating the data distributions so that their \emph{minor}, not \emph{major} axes cross the classifier, reducing the variance of classifier outputs in $\Xhat \mid \alpha$ for each unique $\alpha$. The orientation of $w_\alpha$ can reduce this term by moving the distributions $p(\Xhat \mid \alpha)$ away from the decision boundary (where disagreement in corresponding $Y$ values is lower) as quickly as possible as $\abs{\alpha}$ increases.

\begin{lemma}
    \label{lem:condEnt}
    Let $W = \begin{bmatrix} w_\alpha & W_\beta \end{bmatrix}$, for $w_\alpha \in \R^N$ and $W_\beta \in \R^{N \times (N-1)}$. Suppose that each column $w_i$ of $W$ is bounded by $c>0$, i.e., $\norm{w_i}_2 \le c$. Then under the conditions of Proposition \ref{prop:lingauss/single-hyperplane}, $H(Y \mid \alpha)$ is minimized when $w_\alpha=\pm c \frac{a}{\norm{a}_2}$ and $W_{\beta}^T a = 0$.
\end{lemma}

\begin{proof}
We have $\Xhat = w_\alpha \alpha + W_\beta \beta + \varepsilon$ with $\varepsilon \sim \mathcal{N}(0, \gamma I)$. For fixed $\alpha$, $p(\Xhat \mid \alpha) = \mathcal{N}(w_\alpha \alpha, W_\beta W_\beta^T + \gamma I)$. Defining $U = a^T X$, we have $U \mid \alpha \sim \mathcal{N}(\alpha a^T w_\alpha, a^T W_\beta W_\beta^T a + \gamma \norm{a}_2^2)$. Then,
\begin{align*}
    p(\widehat{Y}=1 \mid \alpha) &= \E_{\Xhat \mid \alpha} [p(\widehat{Y}=1 \mid \Xhat, \alpha)]
    \\ &= \E_{\Xhat \mid \alpha} [p(\widehat{Y}=1 \mid \Xhat)]
    \\ &= \E_{\Xhat \mid \alpha} [\sigma(a^T X)]
    \\ &= \E_{U \mid \alpha} [\sigma(U)]
    \\ &\overset{(\star)}{=} \sigma\left(\frac{\alpha \ip{a}{w_\alpha}}{\sqrt{1+a^T W_\beta W_\beta^T a + \gamma \norm{a}_2^2}}\right),
\end{align*}
where ($\star$) follows from the fact that for $Z \sim \mathcal{N}(\mu, \sigma^2)$, $\E_Z[\sigma(Z)] = \sigma\left(\frac{\mu}{\sqrt{1+ \sigma^2}}\right)$.

We can now evaluate the entropy $H(Y \mid \alpha) = \E_{t \sim \alpha} [H(Y \mid \alpha = t)]$. Again denoting the binary entropy function by $h_b$, we have
\begin{align*}
    H(Y \mid \alpha = t) &= h_b(p(Y=1 \mid \alpha = t)) \\
    &= h_b(\sigma(s))\quad\text{where}~ s \coloneqq \frac{t \ip{a}{w_\alpha}}{\sqrt{1+a^T W_\beta W_\beta^T a + \gamma \norm{a}_2^2}} \\
    &= h_b((\sigma(s) - 0.5) + 0.5).
    \intertext{Let $q \coloneqq p - 0.5$ and define $\widetilde{h_b}(q) = h_b(q + 0.5)$ for $q \in [-0.5,0.5]$ so that $\widetilde{h_b}$ is an even function. Therefore, $\widetilde{h_b}(q) = \widetilde{h_b}(\abs{q})$, and we have $h_b(p) = \widetilde{h_b}(p - 0.5) = \widetilde{h_b}(\abs{p - 0.5})$. Applying this fact yields}
    &= \widetilde{h_b}(\sigma(s) - 0.5) \\
    &= \widetilde{h_b}(\abs{\sigma(s) - 0.5}) \\
    &\overset{(\dagger)}{=} \widetilde{h_b}(\abs{\sigma(\abs{s}) - 0.5})
\end{align*}
where ($\dagger$) follows since $\abs{\sigma(s) - 0.5}$ is an even function of $s$. On $\R_{\ge 0}$ we have that $\widetilde{h_b}(\cdot)$ is a monotonically decreasing function and $\abs{\sigma(\cdot) - 0.5}$ is a monotonically increasing function, and therefore $H(Y \mid \alpha = t) = \widetilde{h_b}(\abs{\sigma(\abs{s}) - 0.5})$ is a monotonically decreasing function of $\abs{s}$ where
\begin{equation}
    \abs{s} = \frac{\abs{t} \abs{\ip{a}{w_\alpha}}}{\sqrt{1+a^T W_\beta W_\beta^T a + \gamma \norm{a}_2^2}}.
    \label{eq:Harg}
\end{equation}
For any value of $t$, it is clear that the expression in \eqref{eq:Harg} is maximized (and therefore $H(Y \mid \alpha = t)$ is minimized) with respect to $w_\alpha$ and $W_\beta$ when both $\abs{\ip{a}{w_\alpha}}$ is maximized and $a^T W_\beta W_\beta^T a$ is minimized. By the Cauchy-Schwarz inequality and from boundedness of the column magnitudes of $W$ by $c$, we have that $\abs{\ip{a}{w_\alpha}}$ is maximized at $w_\alpha=\pm c \frac{a}{\norm{a}_2}$. Since $a^T W_\beta W_\beta^T a \ge 0$, this quadratic term is minimized at $W_\beta^Ta = 0$ in which case $a^T W_\beta W_\beta^T a = 0$.

Since choosing $w_\alpha$ and $W_\beta$ in this way minimizes $H(Y \mid \alpha = t)$ for any $t$, we have that $H(Y \mid \alpha) = \E_{t \sim \alpha}[H(Y \mid \alpha = t)]$ is also minimized with this choice of $w_\alpha$ and $W_\beta$. 
\end{proof}

Since $H(Y)$ is constant for any $W$ (Lemma \ref{lem:ent}), we have that the conditions on $W$ described in Lemma \ref{lem:condEnt} maximize $\CausalEffect=I(\alpha;Y)$. We combine this result with the following lemma to characterize the minimum of the entire objective \eqref{eq:objective}:

\begin{lemma}
    \label{lem:KLgauss}
   Suppose that $\varepsilon<1$, $W \in \R^{N \times N}$, and that $X,\varepsilon,z \sim \mathcal{N}(0,I)$ in $\R^N$. With $U = Wz + \gamma \varepsilon$, $\mathrm{D}_{\mathrm{KL}}(p(X) ~\Vert~ p(U))$ is minimized by any orthogonal $W$ with columns normalized to magnitude $\sqrt{(1-\gamma)}$.
\end{lemma}

\begin{proof}
    Noting that $U \sim \mathcal{N}(0,WW^T + \gamma I)$, we have from a standard result on KL divergence between multivariate normal distributions that
    \begin{equation}
        \argmin_{W} \mathrm{D}_{\mathrm{KL}}(p(X) ~\Vert~ p(U)) = \argmin_{W} \log \abs{WW^T + \gamma I} + \mathrm{tr}((WW^T + \gamma I)^{-1}).
        \label{eq:KLmin}
    \end{equation}
    
    Since $WW^T + \gamma I$ is positive definite, there exists orthogonal $V$ and diagonal $\Lambda$ with positive entries $\{\lambda_i\}_{i=1}^N$ such that $W W^T + \gamma I = V \Lambda V^T$. We then have
    \begin{align}
        \log \abs{WW^T + \gamma I} + \mathrm{tr}((WW^T + \gamma I)^{-1}) &= \log \abs{V \Lambda V^T} + \mathrm{tr}((V \Lambda V^T)^{-1})\notag
        \\ &= \log \abs{\Lambda} + \mathrm{tr}( \Lambda^{-1})\notag
        \\ &= \sum_i \log \lambda_i + \frac{1}{\lambda_i}.
        \label{eq:sumlam}
    \end{align}
    \eqref{eq:sumlam} is minimized at $\lambda_i=1$ for all $i$. Therefore, the minimizer of \eqref{eq:KLmin} is characterized by $WW^T = VV^T - \gamma I = (1-\gamma)I$. Any orthogonal $W$ with column magnitudes equal to $\sqrt{1 - \gamma}$ satisfies this condition.
\end{proof}

Combining these lemmas, consider the solution $w_\alpha = \sqrt{1-\gamma} \frac{a}{\norm{a}_2}$, and $W_\beta$ with orthogonal, $\sqrt{1-\gamma}$-norm columns satisfying $W_\beta^T a = 0$. From Lemma \ref{lem:condEnt} we have that this solution minimizes $H(Y \mid \alpha)$ within the class of $N \times N$ matrices whose column magnitudes are bounded by $\sqrt{1-\gamma}$. Combined with the invariance of $H(Y)$ to $W$ (Lemma \ref{lem:ent}), we have that $I(\alpha;Y)$ is maximized by this choice of $w_\alpha$ and $W_\beta$. From Lemma \ref{lem:KLgauss} we have that this solution also minimizes $\DataFidelity=\mathrm{D}_{\mathrm{KL}}(p(X) ~\Vert~ p(U))$, and thus this solution minimizes the objective \eqref{eq:objective} for any $\lambda>0$.

\subsection{Proof of Proposition \ref{prop:information-flow-relationships}}
\label{sec:supp/proofs/information-flow-relationships}

Proposition \ref{prop:information-flow-relationships} states the relationships between information flow-based objectives depicted graphically in Figure \ref{fig:information-flow-relationships}.

\begin{proof}[Proof of (a)]
    We have that
    \begin{align*}
        I(Y; \alpha) &= \frac1K \sum_{i=1}^K I(Y; \alpha_1, \dots, \alpha_K) \\
        &= \frac1K \sum_{i=1}^K \left[ I(Y; \alpha_i) + I(Y; \alpha_{\neg i} \mid \alpha_i) \right] \\
        &= \frac1K \sum_{i=1}^K I(Y; \alpha_i) + \frac1K \sum_{i=1}^K I(Y; \alpha_{\neg i} \mid \alpha_i).
    \end{align*}
\end{proof}

\begin{proof}[Proof of (b)]
    First, note that
    \begin{align*}
        I(X;Y \mid Z,W) &= \int_x \int_y \int_z \int_w p(x,y,z,w) \log \frac{p(x,y \mid z,w)}{p(x \mid z,w) p(y \mid z,w)} dx dy dz dw \\
        &= \int_x \int_y \int_z \int_w p(x,y,z,w) \log \frac{p(x,y,z,w) / p(z,w)}{p(x,z,w) / p(z,w) p(y,z,w) / p(z,w)} dx dy dz dw \\
        &= \int_x \int_y \int_z \int_w p(x,y,z,w) \log \frac{p(x,y,z,w) p(z,w) p(x,w)}{p(x,z,w) p(y,z,w) p(x,w)} dx dy dz dw \\
        &= \int_x \int_y \int_z \int_w p(x,y,z,w) \log \frac{p(x,y,z \mid w) p(z \mid w) p(x \mid w)}{p(x,z \mid w) p(y,z \mid w) p(x \mid w)} dx dy dz dw \\
        &= \int_x \int_y \int_z \int_w p(x,y,z,w) \left( \log \frac{p(x,y,z \mid w)}{p(x \mid w) p(y,z \mid w)} \right. \\
        &\qquad\qquad \left. - \log \frac{p(x,z \mid w)}{p(x \mid w) p(z \mid w)} \right) dx dy dz dw \\
        &= I(X; Y, Z \mid W) - \E_{Y}\left[I(X; Z \mid W)\right] \\
        &= I(X; Y, Z \mid W) - I(X; Z \mid W).
    \end{align*}
    Applying this identity,
    \begin{align*}
        I(Y; \alpha \mid \beta) &= \frac1K \sum_{i=1}^K I(Y; \alpha_i, \alpha_{\neg i} \mid \beta) \\
        &= \frac1K \sum_{i=1}^K \left[ I(Y; \alpha_i \mid \alpha_{\neg i}, \beta) + I(Y; \alpha_{\neg i} \mid \beta) \right] \\
        &= \frac1K \sum_{i=1}^K I(Y; \alpha_i \mid \alpha_{\neg i}, \beta) + \frac1K \sum_{i=1}^K I(Y; \alpha_{\neg i} \mid \beta).
    \end{align*}
\end{proof}

\begin{proof}[Proof of (c)]
    We have that
    \begin{align*}
        I(Y; \alpha \mid \beta) &= \int_{Y} \int_{\alpha} \int_{\beta} p(Y, \alpha, \beta) \log \frac{p(Y, \alpha \mid \beta)}{p(Y \mid \beta) p(\alpha \mid \beta)} dY d\alpha d\beta \\
        &\overset{(\star)}{=} \int_{Y} \int_{\alpha} \int_{\beta} p(Y, \alpha, \beta) \log \frac{p(\beta \mid Y, \alpha) p(Y, \alpha)}{p(\beta)} \frac{p(\beta)}{p(\beta \mid Y) p(Y)} \frac{p(\beta)}{p(\beta \mid \alpha) p(\alpha)} dY d\alpha d\beta \\
        &\overset{(\star\star)}{=} \int_{Y} \int_{\alpha} \int_{\beta} p(Y, \alpha, \beta) \log \frac{p(\beta \mid Y, \alpha) p(Y, \alpha)}{p(\beta)} \frac{p(\beta)}{p(\beta \mid Y) p(Y)} \frac{p(\beta)}{p(\beta) p(\alpha)} dY d\alpha d\beta \\
        &= \int_{Y} \int_{\alpha} \int_{\beta} p(Y, \alpha, \beta) \log \frac{p(Y,\alpha)}{p(Y)p(\alpha)} \frac{p(\beta \mid Y, \alpha)}{p(\beta \mid Y)} dY d\alpha d\beta \\
        &= \int_{Y} \int_{\alpha} \int_{\beta} p(Y, \alpha, \beta) \log \frac{p(Y,\alpha)}{p(Y)p(\alpha)} \frac{p(\beta \mid Y, \alpha) p(\alpha \mid Y)}{p(\beta \mid Y) p(\alpha \mid Y)} dY d\alpha d\beta \\
        &= \int_{Y} \int_{\alpha} \int_{\beta} p(Y, \alpha, \beta) \log \frac{p(Y,\alpha)}{p(Y)p(\alpha)} \frac{p(\alpha,\beta \mid Y)}{p(\alpha \mid Y) p(\beta \mid Y)} dY d\alpha d\beta \\
        &= \int_{Y} \int_{\alpha} \int_{\beta} p(Y, \alpha, \beta) \left( \log \frac{p(Y,\alpha)}{p(Y)p(\alpha)} + \log \frac{p(\alpha,\beta \mid Y)}{p(\alpha \mid Y) p(\beta \mid Y)} \right) dY d\alpha d\beta \\
        &= I(Y; \alpha) + I(\alpha ; \beta \mid Y),
    \end{align*}
    where ($\star$) follows from Bayes' rule and ($\star\star$) follows from the independence of $\alpha$ and $\beta$ in our model.
\end{proof}

\begin{proof}[Proof of (d)]
    Similar to (c).
\end{proof}
\clearpage

\section{Sample-based estimate of causal influence}
\label{sec:supp/sampling}
Here we detail the sampling procedure for approximating the causal objective in \eqref{eq:causal-effect}. (The variants described in Appendix \ref{sec:supp/causal-obj-variants} can be approximated in similar fashion.) We have
\begin{equation*}
    \CausalEffect(\alpha; Y) = I(\alpha; Y) = \int_{\alpha} p(\alpha) \left(\sum_{y} p(y \mid \alpha)\log p(y \mid \alpha)\right)d\alpha - \sum_{y} p(y)\log p(y)
\end{equation*}
where
\begin{equation}
    p(y \mid \alpha) = \int_{\beta} \int_x p(y \mid x) p(x \mid \alpha, \beta) p(\beta) dx d\beta
    \label{eq:sample-estimate-pygivenalpha}
\end{equation}
and
\begin{equation}
    p(y) = \int_{\alpha,\beta} \int_x p(y \mid x) p(x \mid \alpha,\beta) p(\alpha) p(\beta) dx d\alpha d\beta.
    \label{eq:sample-estimate-py}
\end{equation}

\begin{algorithm}[t]
\caption{Sample-based estimate of $\CausalEffect(\alpha;Y)$}
\label{alg:sample-causal-effect}
\begin{algorithmic}
    \REQUIRE number of samples $N_\alpha$ and $N_\beta$, number of latent factors $K$ and $L$, number of classes $M$
    \STATE $I \leftarrow 0$
    \STATE $\bm{q}_y \leftarrow \mathrm{zeros}(M)$ %\COMMENT{$\bm{q}_y$: sample-based estimate of $p(y)$}
    \FOR {$i=1$ \TO $N_\alpha$}
        \STATE $\alpha \leftarrow$ $K$-dimensional vector sampled from $\mathcal{N}(0,I)$
        \STATE $\bm{p}_{y \mid \alpha} \leftarrow \mathrm{zeros}(M)$ %\COMMENT{$\bm{p}_{y \mid \alpha}$: sample-based estimate of $p(y \mid \alpha)$}
        \FOR{$j=1$ \TO $N_\beta$}
            \STATE $\beta \leftarrow$ $L$-dimensional vector sampled from $\mathcal{N}(0,I)$
            \STATE $x \leftarrow$ sample from $p(x \mid \alpha, \beta)$
            \STATE $\bm{p}_{y \mid \alpha} \leftarrow \bm{p}_{y \mid \alpha} + \frac{1}{N_\beta} p(y\mid x)$ (where $p(y \mid x) \in \R^M$ is the classifier probability for each class)
            %\FOR{$y = 1$ \TO $\abs{\mathcal{Y}}$}
            %    \STATE $\rho \leftarrow \frac{1}{N_x}\sum_{x \in \mathcal{X}}p(y\ \mid x)$
            %    \STATE $I_m \leftarrow I_m + \frac{1}{N_i} \rho \log \rho$
            %    \STATE $\mathbf{p}[y] \leftarrow \mathbf{p}[y] + \frac{1}{N_i}\rho$ 
            %\ENDFOR
        \ENDFOR
        \STATE $I \leftarrow I + \frac{1}{N_\alpha}\sum_{m=1}^M \bm{p}_{y \mid \alpha}[m] \log \bm{p}_{y \mid \alpha}[m]$
        \STATE $\bm{q}_y \leftarrow \bm{q}_y + \frac{1}{N_\alpha}\bm{p}_{y \mid \alpha}$
    \ENDFOR
    \STATE $I \leftarrow I - \sum_{m=1}^M \bm{q}_y[m] \log \bm{q}_y[m]$
    \ENSURE $I$ (sample-based estimate of $I(\alpha;Y)$)
\end{algorithmic}
\end{algorithm}

For fixed $\alpha$, we approximate \eqref{eq:sample-estimate-pygivenalpha} with $N_x$ and $N_\beta$ samples of $x$ and $\beta$, respectively, as
\begin{equation*}
    p(y \mid \alpha) \approx \frac{1}{N_\beta N_x} \sum_{j=1}^{N_\beta} \sum_{n=1}^{N_x} p(y \mid x^{(n)}),
\end{equation*}
where each $x^{(n)} \sim p(x \mid \alpha, \beta^{(j)})$ and $\beta^{(j)} \sim p(\beta)$. Similarly, we approximate \eqref{eq:sample-estimate-py} with $N_x$, $N_\alpha$, and $N_\beta$ samples of $x$, $\alpha$, and $\beta$, respectively, as
\begin{equation*}
    p(y) \approx \frac{1}{N_\alpha N_\beta N_x} \sum_{j=1}^{N_\beta} \sum_{i=1}^{N_\alpha} \sum_{n=1}^{N_x} p(y \mid x^{(n)}),
\end{equation*}
where each $x^{(n)} \sim p(x \mid \alpha^{(i)},\beta^{(j)})$, $\alpha^{(i)} \sim p(\alpha)$, and $\beta^{(j)} \sim p(\beta)$. Therefore,
\begin{equation*}
\begin{split}
    I(\alpha_i;y) \approx \frac{1}{N_\alpha N_\beta N_x} \left[ \sum_{i=1}^{N_\alpha} \sum_y \left( \sum_{j=1}^{N_\beta} \sum_{n=1}^{N_x} p(y \mid x^{(n)}) \right) \log \left( \frac{1}{N_\beta N_x} \sum_{j=1}^{N_\beta} \sum_{n=1}^{N_x} p(y \mid x^{(n)}) \right) \right. \\
    \qquad - \left. \sum_y \left( \sum_{j=1}^{N_\beta} \sum_{i=1}^{N_\alpha} \sum_{n=1}^{N_x} p(y \mid x^{(n)}) \log \left( \frac{1}{N_\alpha N_\beta N_x} \sum_{j=1}^{N_\beta} \sum_{i=1}^{N_\alpha} \sum_{n=1}^{N_x} p(y \mid x^{(n)}) \right) \right) \right]
\end{split}
\end{equation*}
where each $x^{(n)} \sim p(x \mid \alpha^{(i)},\beta^{(j)})$, $\alpha^{(i)}\sim p(\alpha)$, and $\beta^{(j)}\sim p(\beta)$.

The complete procedure is described algorithmically in Algorithm \ref{alg:sample-causal-effect} with $N_x = 1$.
\clearpage

\section{VAE experimental details and additional results}
\label{sec:supp/vae-details}
\subsection{Details and additional results for MNIST experiments}
\label{sec:supp/vae-details/mnist}

All experiments were run using a single Nvidia GeForce GTX 1080 GPU. The traditional MNIST training set was split into training and validation sets composed of the first 50,000 and remaining 10,000 images, respectively. The testing set was the same as the traditional MNIST testing set, composed of 10,000 images. These sets were down-selected to include only samples with the labels of interest. Input images were scaled so that the network inputs are in $[0,1]^{28 \times 28}$.

\begin{table}[t]
    \centering
    \begin{tabular}{|c|} 
    \hline
    Classifier Architecture \\ 
        \hline
        Input (28$\times$28)   \\ 
        Conv2 (32 channels, 3$\times$3 kernels, stride 1, pad 0) \\
        ReLU \\
        Conv2 (64 channels, 3$\times$3 kernels, stride 1, pad 0) \\
        ReLU \\
        MaxPool (2$\times$2 kernel) \\
        Dropout ($p = 0.5$) \\
        Linear (128 units) \\ 
        ReLU \\
        Dropout ($p = 0.5$) \\
        Linear ($M$ units) \\
        Softmax \\
        \hline
    \end{tabular}
    \caption{Network architecture for MNIST Classifier}
    \label{tab:mnist-classifier-details}
\end{table}

The network architecture for the classifier used in the MNIST experiments is shown in Table~\ref{tab:mnist-classifier-details} where $M$, the number of class outputs, varies depending on the classification task. The classifier was trained with a batch size of 64 and a stochastic gradient descent optimizer with momentum 0.5 and learning rate 0.1. The 3/8 classifier was trained for 20 epochs and the 1/4/9 classifier was trained for 30 epochs.  The test accuracy of the classifier trained on both the 3/8 and 1/4/9 datasets was 99.6\%.

\begin{figure}
    \centering
    \includegraphics[width=\textwidth]{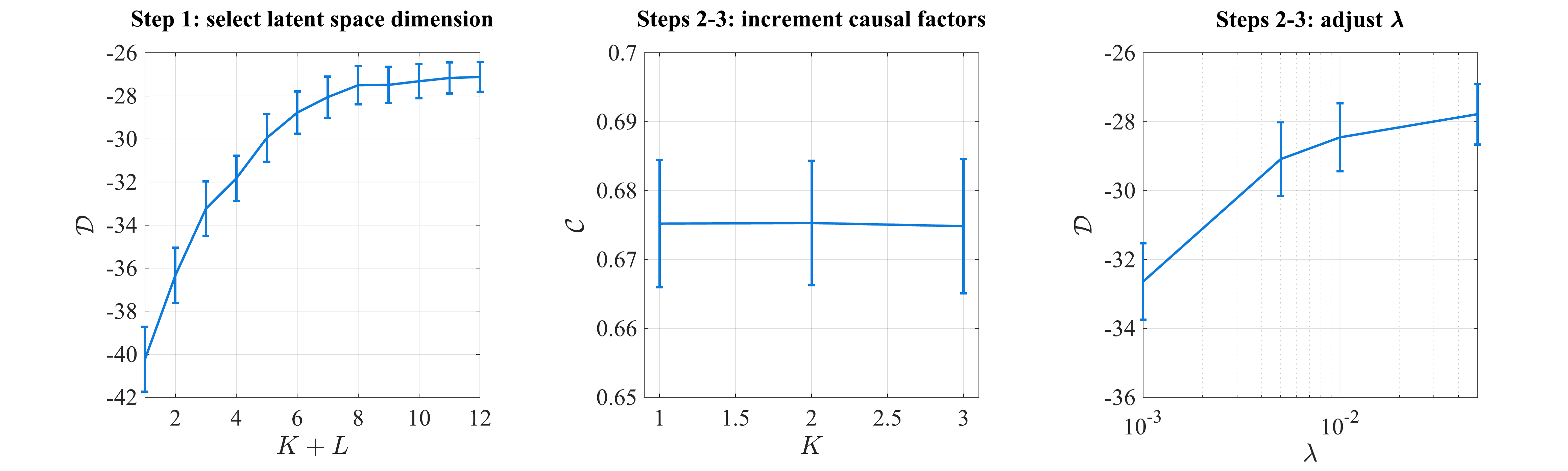}
    \caption{Partial details of parameter tuning procedure used to select $K$, $L$, and $\lambda$ for explaining MNIST 3/8 classifier using Algorithm \ref{alg:parameter-selection}. \emph{Left:} In Step 1 we select the total number of latent factors $K+L$ needed to adequately represent the data distribution. \emph{Center:} In Steps 2-3 we iteratively convert noncausal latent factors to causal latent factors until $\CausalEffect$ plateaus. \emph{Right:} After each increment of $K$, we adjust $\lambda$ to approximately achieve the value of $\DataFidelity$ from Step 1.}
    \label{fig:tuning-mnist}
\end{figure}

\begin{table}[t]
    \centering
    \begin{tabular}{|c|c|} 
        \hline
        VAE Encoder Architecture & VAE Decoder Architecture  \\ 
        \hline
        Input (28$\times$28) & Input ($K+L$) \\ 
        Conv2 (64 chan., 4$\times$4 kernels, stride 2, pad 1) & Linear (3136 units)  \\
        ReLU & ReLU \\
        Conv2 (64 chan., 4$\times$4 kernels, stride 2, pad 1) & Conv2Transp (64 chan., 4$\times$4 kernels, stride 1, pad 1) \\ 
        ReLU & ReLU \\
        Conv2 (64 chan., 4$\times$4 kernels, stride 1, pad 0) & Conv2Transp (64 chan., 4$\times$4 kernels, stride 2, pad 2) \\
        ReLU & ReLU \\
        Linear ($K+L$ units for both $\mu$ and $\sigma$) & Conv2Transp (1 chan., 4$\times$4 kernel, stride 2, pad 1) \\
        & Sigmoid \\
        \hline
    \end{tabular}
    \caption{VAE network architecture used for MNIST and Fashion MNIST experiments.}
    \label{tab:mnist-vae-details}
\end{table}

The VAE architecture used to learn the generative map $g$ is shown in Table \ref{tab:mnist-vae-details}. The objective \eqref{eq:objective} was maximized with 8000 training steps, batch size 64, and learning rate $5 \times 10^{-4}$. At each training step, the causal influence term \ref{eq:causal-effect} was estimated using the sampling procedure in Appendix \ref{sec:supp/sampling} with $N_\alpha = 100$ and $N_\beta = 25$. For experiments with digits 3 and 8, we selected $K = 1$, $L = 7$, and $\lambda = 0.05$ using the parameter selection procedure in Algorithm \ref{alg:parameter-selection}; Figure \ref{fig:tuning-mnist} shows intermediate results from this procedure.

\begin{figure}[t]
    \centering
    \includegraphics{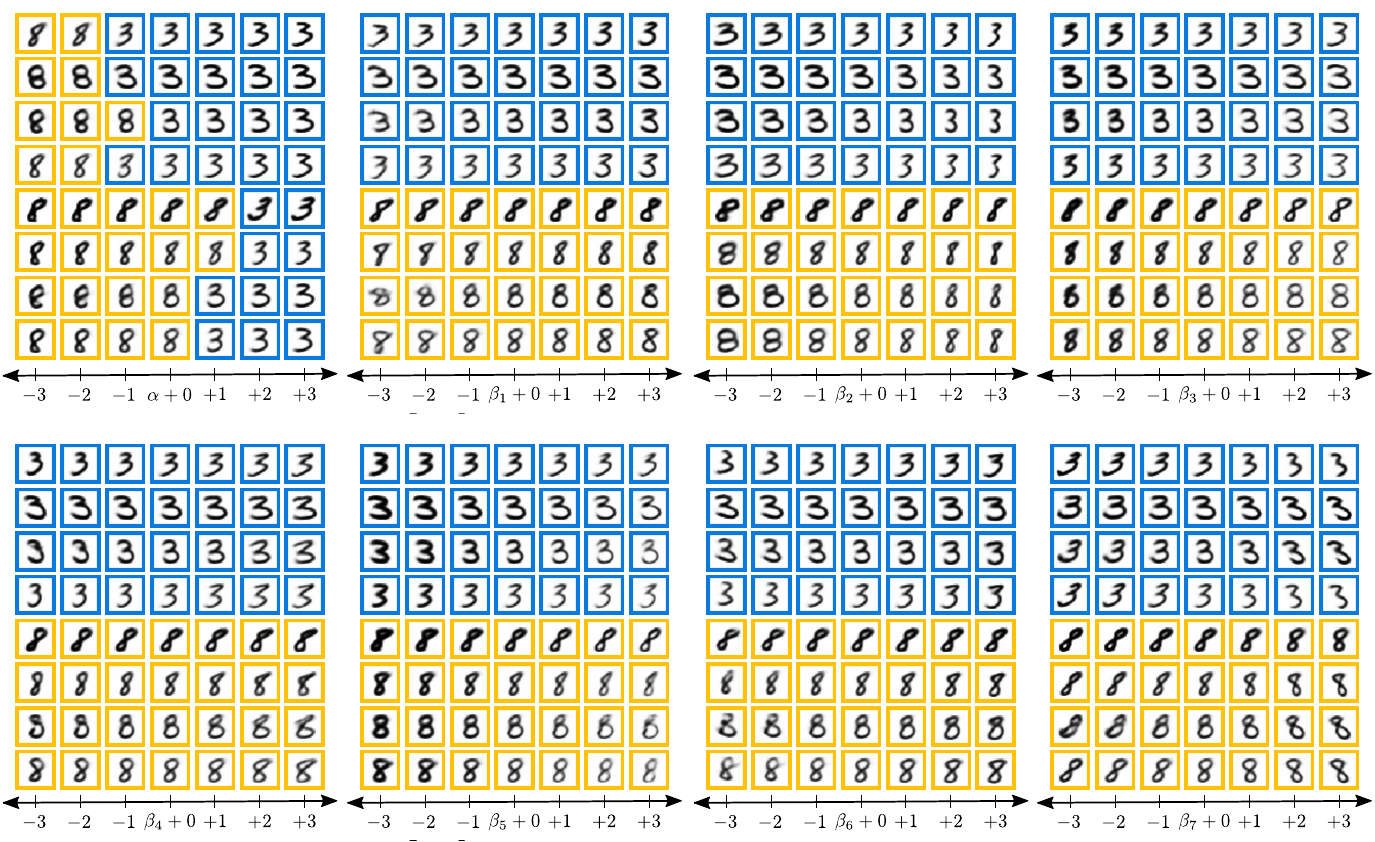}
    \caption{Visualizations for learned latent factors for MNIST 3/8 classifier. Images in the center column of each grid are reconstructed samples from the validation set; moving left or right in each row shows $g(\alpha,\beta)$ as a single latent factor is varied. This plot shows the complete results from Figure \ref{fig:mnist-qual}; it includes sweeps for two additional samples and visualizations of all $L=7$ noncausal factors.}
    \label{fig:mnist-qual-details}
\end{figure}

Figure \ref{fig:mnist-qual-details} shows additional results for the experiment of Figure \ref{fig:mnist-qual}, which visualizes the learned latent factors that explain the MNIST 3/8 classifier. Here we show latent factor sweeps from this experiment with additional data samples and all $K+L=8$ latent factors.

\begin{figure}
    \centering
    \includegraphics{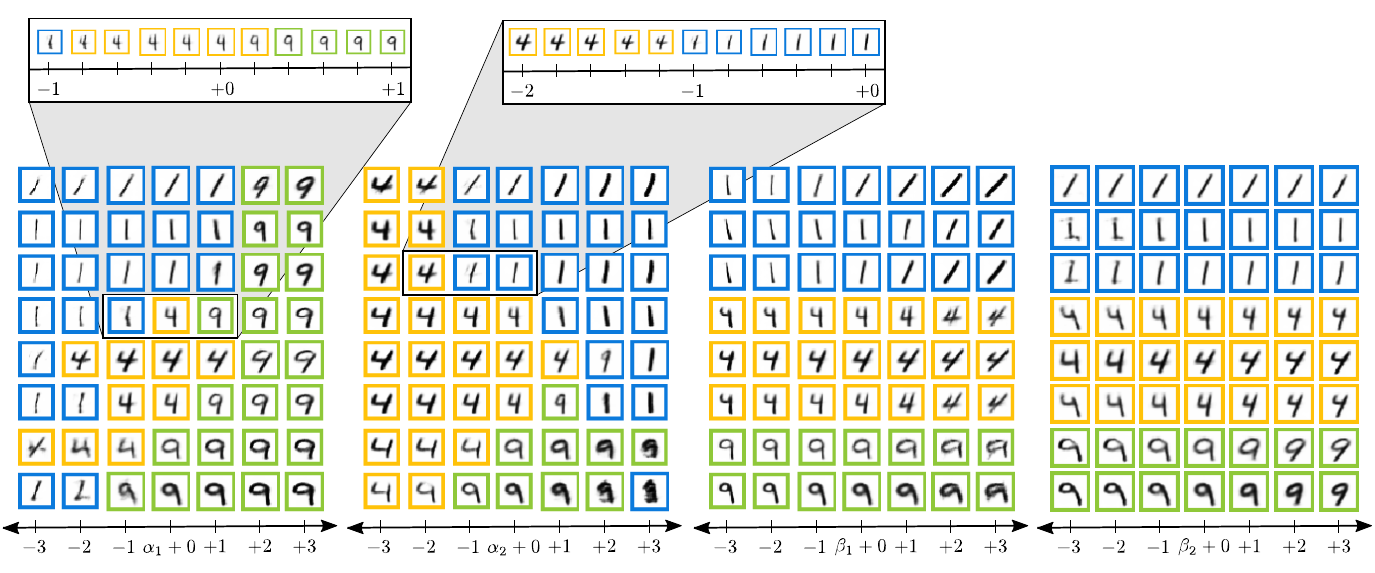}
    \caption{Visualizations of learned latent factors for MNIST 1/4/9 classifier. Images in the center column of each grid are reconstructed samples from the validation set; moving left or right in each row shows $g(\alpha,\beta)$ as a single latent factor is varied. Varying the causal factors $\alpha_1$ and $\alpha_2$ control aspects that affect the classifier output (colored borders); varying the noncausal factors $\beta_1$ and $\beta_2$ affect only stylistic aspects such as rotation and thickness.}
    \label{fig:mnist-149-qual}
\end{figure}

\begin{figure}
    \centering
    \includegraphics{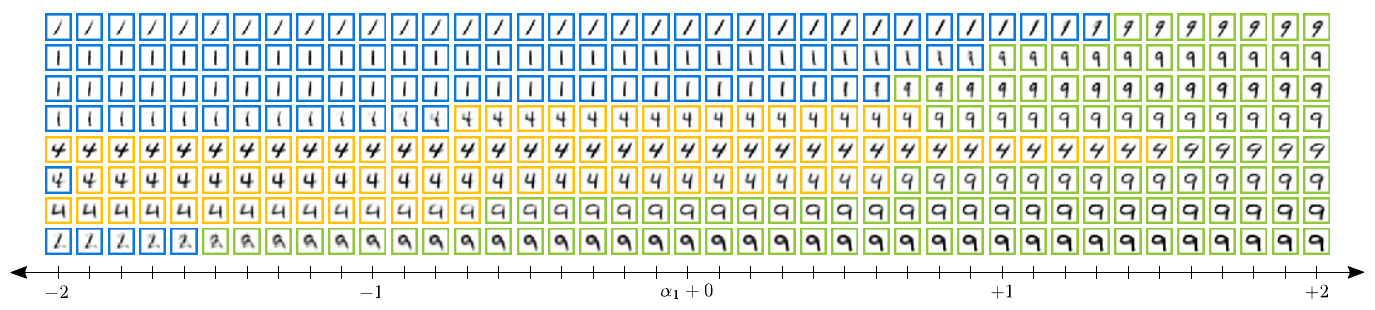}
    \caption{High-resolution transition regions of the first causal factor in explaining the MNIST 1/4/9 classifier. Visualizing high-resolution latent factor sweeps can allow a practitioner to more easily identify which data features correspond to each underlying factor. For example, one can observe in the second row from the bottom how increasing $\alpha_1$ causes the left branch of the digit `4' to smoothly transition into completing the loop of the digit `9' while the digit stem remains fixed.}
    \label{fig:mnist-149-qual-zoomed}
\end{figure}

Figure \ref{fig:mnist-149-qual} shows an explanation of the same classifier architecture detailed in Table~\ref{tab:mnist-classifier-details} trained on the MNIST digits 1, 4, and 9. We use the VAE architecture of Table \ref{tab:mnist-vae-details} with $K = 2$ causal factors, $L = 2$ noncausal factors, and $\lambda = 0.1$, and estimated the causal influence portion of the objective using the sampling procedure in Appendix \ref{sec:supp/sampling} with $N_\alpha = 75$ and $N_\beta = 25$. While the factor sweeps in Figure \ref{fig:mnist-149-qual} provide a high-level indication of the data features each factor corresponds to, a practitioner may also wish to visualize the fine-grained transitions between each class. This can be achieved by sweeping each factor on a finer scale, as visualized by the zoomed in regions of Figure \ref{fig:mnist-149-qual} as well as the more comprehensive sweeps in Figure \ref{fig:mnist-149-qual-zoomed}.

\subsection{Details and additional results for comparison experiments}
\label{sec:supp/vae-details/comparison}

Figure \ref{fig:comparison} compares our latent factor-based local explanations to the local explanations of four popular explanation methods. We generate explanations of the same CNN classifier trained on MNIST 3 and 8 digits described in Appendix \ref{sec:supp/vae-details/mnist}. The data samples explained in Figure \ref{fig:comparison} are the first example of each class in the MNIST validation set.

\textbf{Implementation details of other methods.} The following procedures were used to generate the results for LIME, DeepSHAP, IG, and L2X shown in Figure \ref{fig:comparison} (left):
\begin{itemize}
    \item \textbf{LIME \cite{ribeiro2016why}.} The LIME framework trains a sparse linear model using superpixel features. Following the recommendation in the authors' code, we generate superpixels using the Quickshift segmentation algorithm from scikit-image with kernel size 1, maximum distance 200, and color/image-space proximity ratio 0 (as the MNIST digits are grayscale). The LIME local approximation is fit using the default kernel width of 0.25, 10,000 samples, and $K=10$ features. Figure \ref{fig:comparison} show superpixels identified as contributing positively (red) and or negatively (blue) to the classification decision.
    \item \textbf{DeepSHAP \cite{lundberg2017unified}.} The DeepSHAP method uses the structure of the classifier network to efficiently approximate Shapley values, a game-theoretic formulation for how to optimally distribute rewards to players of a cooperative game. The Shapley values displayed in Figure \ref{fig:comparison} can be interpreted as the (averaged) importance of each pixel for explaining the difference between $f(x)$ and $\E_{x \sim X}[f(x)]$. We train the explanation model using 1000 randomly chosen samples from the training set. The DeepSHAP method produces explanations for each possible class; we display the Shapley values corresponding to the classifier class (i.e., the top image shows the explanation for ground truth class 3 and the bottom image shows the explanation for ground truth class 8).
    \item \textbf{IG \cite{sundararajan2017axiomatic}.} The integrated gradients (IG) method integrates the gradient of the classifier probabilities with respect to the input as the input changes from a ``baseline.'' We use an all-zero image as the baseline and the trapezoid rule with 50 steps to approximate the integral. The output in Figure \ref{fig:comparison} shows the integrated gradient explanation for each input image.
    \item \textbf{L2X \cite{chen2018learning}.} The learning to explain (L2X) algorithm learns a mask of features $S$ that (approximately) maximizes $I(Y ; X \odot S)$. Following \cite[Sec.~4.3]{chen2018learning}, we find a mask with $k = 4$ active superpixels, each of size $4 \times 4$. The neural network parameterizing the ``explainer'' model $p(S \mid X)$ consists of two convolutional layers ($32$ filters of size $2 \times 2$ each with relu activation, each followed by a max pooling layer with a $2 \times 2$ pool size), followed by a single $2 \times 2$ convolutional filter. This explainer network learns a $7 \times 7$ mask, with each element corresponding to a $2 \times 2$ superpixel in data space. The neural network parameterizing the variational bound $q(Y \mid X \odot S)$ consists of two convolutional layers, each containing $32$ filters of size $2 \times 2$, using relu activation, and followed by a max pooling layer with $2 \times 2$ pool size; followed by a dense layer. The networks parameterizing $p(S \mid X)$ and $q(Y \mid X \odot S)$ were trained together with $10$ epochs of the 9943 MNIST training samples of $3$'s and $8$'s and the outputs $Y$ of the convolutional neural network classifier described in Appendix \ref{sec:supp/vae-details/mnist}.
\end{itemize}

\begin{figure}
    \centering
    \includegraphics{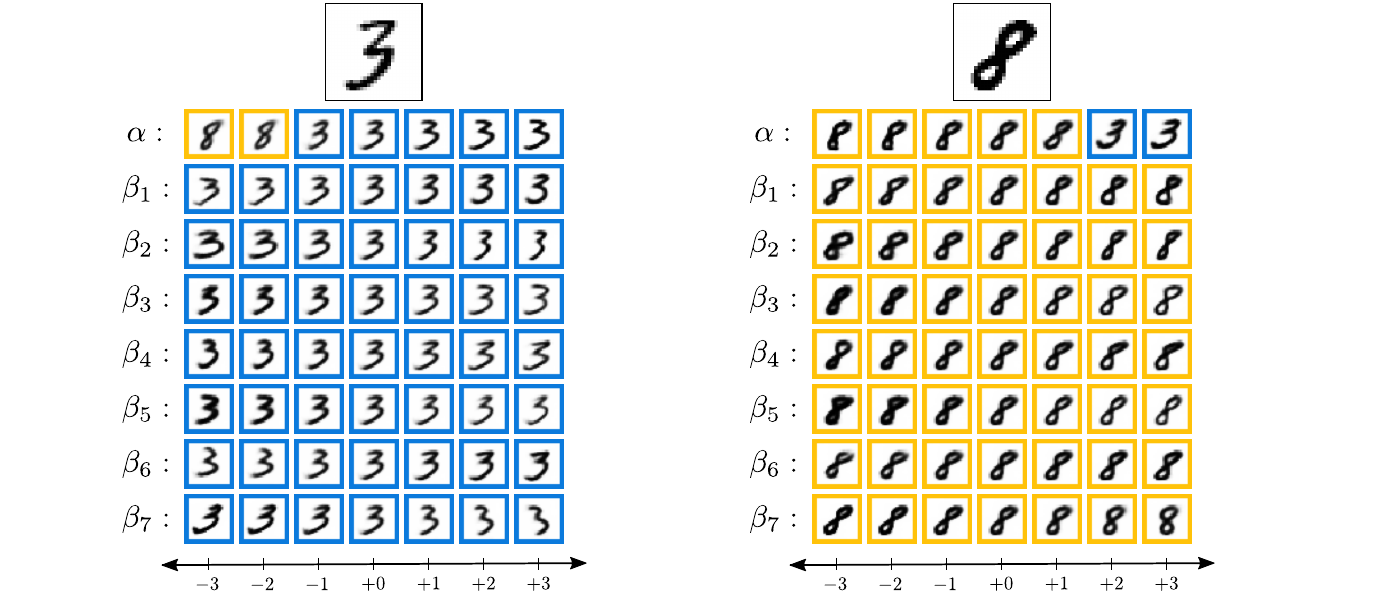}
    \caption{Complete results for local explanations of `3' and `8' from Figure \ref{fig:comparison}. Our explanations are able to differentiate causal aspects (pixels that define 3 from 8) from purely stylistic aspects. Only the causal factor $\alpha$ controls changes in data space that result in a change in classifier output.}
    \label{fig:comparison-detailed}
\end{figure}

\textbf{Complete results for our method.} In Figure \ref{fig:comparison} (right) we show only latent factor sweeps for the the causal factor $\alpha$ and a single noncausal factor $\beta_7$. Figure \ref{fig:comparison-detailed} shows complete local explanations with each noncausal factor. Our explanations use the VAE framework described in Appendix \ref{sec:supp/vae-details/mnist}.

\subsection{Details and additional results for fashion MNIST experiments}
\label{sec:supp/vae-details/fmnist}

Our training set was the same as the traditional Fashion MNIST training set, composed of 60,000 images. The Fashion MNIST testing set was split into validation and testing sets composed of the first 6,000 and last 4,000 images, respectively. These sets were down-selected to include only samples with the labels of interest --- in our experiment, classes 0 (`t-shirt/top'), 3 (`dress'), and 4 (`coat'). Input images were scaled so that the input images were in $[0,1]^{28 \times 28}$. 

The same classifier architecture described in Table \ref{tab:mnist-classifier-details} was used in this experiment. The classifier was trained with 50 epochs, a batch size of 64, a stochastic gradient descent optimizer with momentum 0.5 and learning rate 0.1. Because the classes used (`t-shirt/top,' `dress,' and `coat') are similar, this classifier task is more challenging than the MNIST digit classification task; the test accuracy of the classifier was 95.2\%.

\begin{figure}
    \centering
    \includegraphics[width=\textwidth]{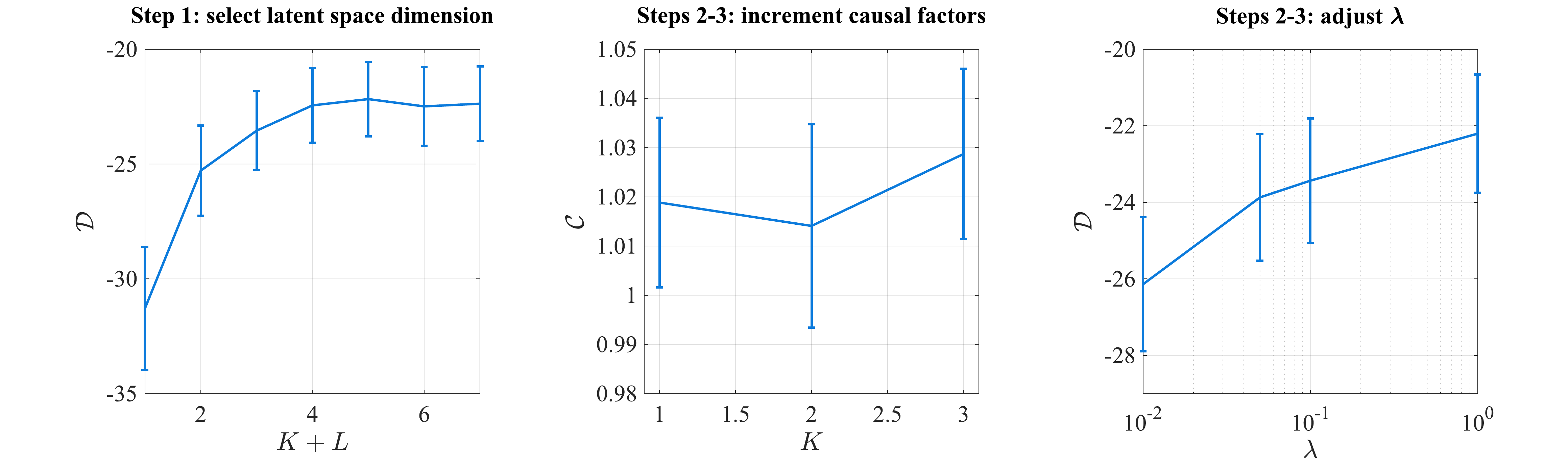}
    \caption{Partial details of parameter tuning procedure used to select $K$, $L$, and $\lambda$ for explaining a classifier trained on classes 0, 3, and 4 of the fashion MNIST dataset using Algorithm \ref{alg:parameter-selection}. \emph{Left:} in Step 1 we select the total number of latent factors $K+L$ needed to adequately represent the data distribution. \emph{Center:} In Steps 2-3 we iteratively convert noncausal latent factors to causal latent factors until $\CausalEffect$ (shown in nats) plateaus. \emph{Right:} After each increment of $K$, we adjust $\lambda$ to approximately achieve the value of $\DataFidelity$ from Step 1.}
    \label{fig:tuning-fmnist}
\end{figure}

The same VAE architecture described in Table \ref{tab:mnist-vae-details} was used to learn the generative map $g$. The objective \eqref{eq:objective} was maximized with 8000 training steps, batch size 32, and learning rate $10^{-4}$. At each training step, the causal influence term \eqref{eq:causal-effect} was estimated using the sampling procedure in Appendix \ref{sec:supp/sampling} with $N_\alpha = 100$ and $N_\beta = 25$. Using the parameter selection procedure in Algorithm \ref{alg:parameter-selection}, we selected $K = 2$, $L = 4$, and $\lambda = 0.05$; Figure \ref{fig:tuning-fmnist} shows intermediate results from this procedure.

Figure \ref{fig:fmnist-details} contains the complete results from the experiment in Figure \ref{fig:quantitative} (right), showing a complete visualization of the global explanation learned for this classifier.

\begin{figure}
    \centering
    \includegraphics{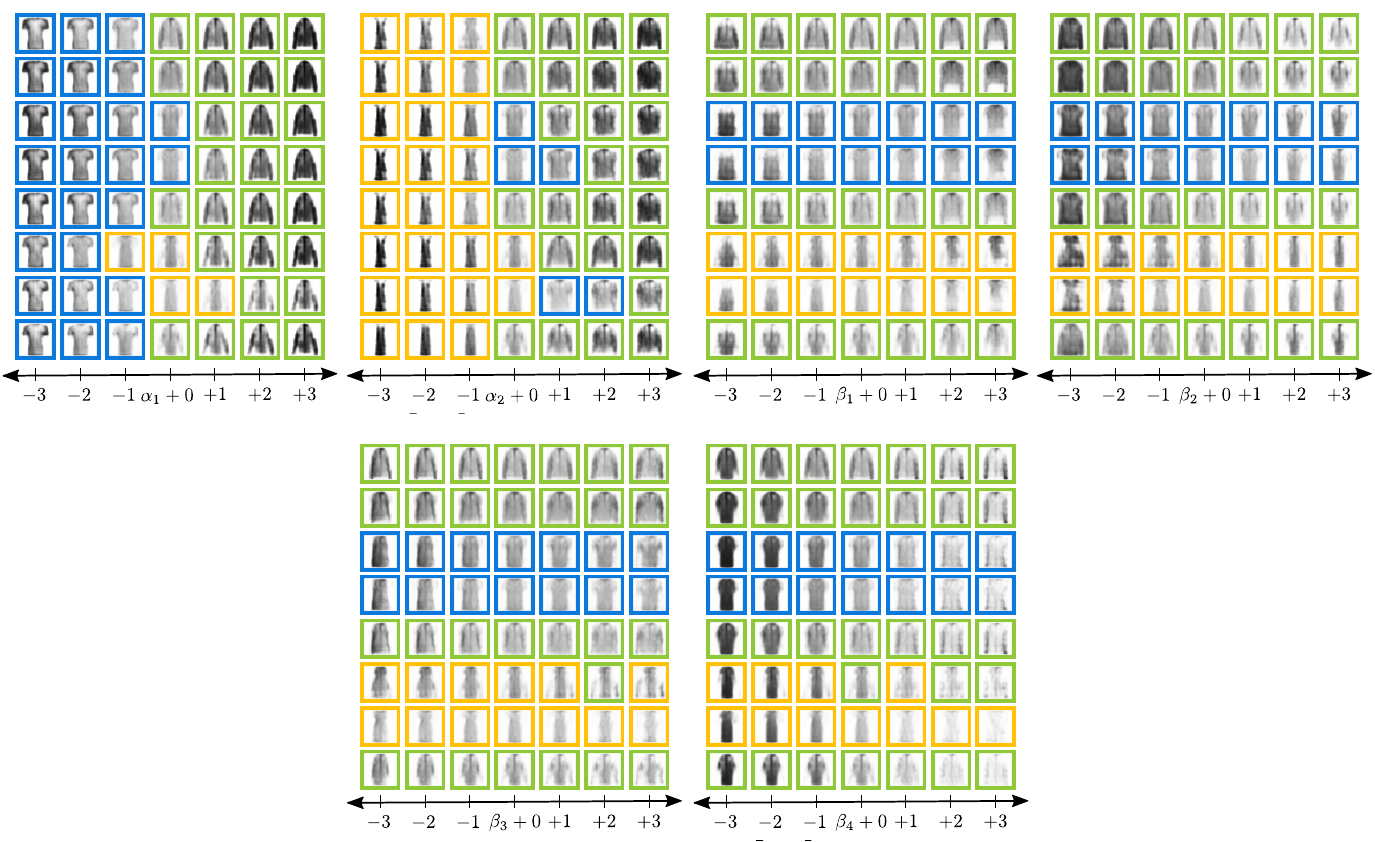}
    \caption{Visualizations of learned latent factors for Fashion MNIST classifier trained on classes `t-shirt-top,' `dress,' and `coat.' Images in the center column column of each grid are reconstructed samples from the validation set; moving left or right in each row shows $g(\alpha,\beta)$ as a single latent factor is varied. This plot shows the complete results from Figure \ref{fig:quantitative} (right); it includes sweeps for two additional samples and visualizations of all $K+L = 6$ latent factors.}
    \label{fig:fmnist-details}
\end{figure}
\clearpage

\section{Selecting generative model capacity}
\label{sec:supp/vae-capacity}
One practical decision to make when constructing explanations using our method is selecting the \emph{capacity} of the generative model $g$. Set too low, the generative model will have insufficient capacity to represent the data distribution and classifier, reducing the quality of the explanation. Set too high, the generative model will require a more time- and energy-intensive training procedure.

We can use results from \cite{feder1994relations} to bound the capacity mismatch of our explainer (i.e., explainer error in predicting classifier outputs) with the $I(\alpha;Y)$ part of our objective. In practice, this result means that a sufficiently large value of $I(\alpha;Y)$ serves as a certificate that the explainer complexity is sufficient to explain the classifier. Below, we show details of this analysis and empirically demonstrate how $I(\alpha;Y)$ can be used to select an architecture with sufficient capacity.

\subsection{$I(\alpha;Y)$ serves as a certificate of sufficient explainer capacity}

One reasonable measure for the quality of an explanation method is how accurately the black-box's classifications can be predicted from the explanation alone. If this prediction is accurate, then in a predictive sense the explanation has captured the relevant information about the classifier's behavior. In our model, the estimator that minimizes prediction error is the MAP estimate of the classifier's output from $p(Y \mid \alpha)$, where $p(Y \mid \alpha)$ is determined by marginalizing $p(Y \mid X)p(X \mid \alpha,\beta) p(\beta)$ over $\beta$ and $X$. As we show below, we can upper bound the error of this MAP estimator \emph{directly} by the causal effect $I(\alpha;Y)$ of $\alpha$ on $Y$, the quantity our method explicitly optimizes.

Specifically, let $\pi(Y \mid \alpha) \coloneqq \int_\alpha [1 - \max_{y}\, p(y \mid \alpha)]\,p(\alpha) d\alpha$ denote the expected error of this MAP estimator, averaged over the prior distribution on causal factor $\alpha$. From \cite{feder1994relations}, we have
\[\phi^*(\pi(Y \mid \alpha)) \le H(Y \mid \alpha),\]
where $H(Y \mid \alpha)$ is the conditional entropy of $Y$ given $\alpha$, and $\phi^*$ is a monotonically increasing, invertible function. Define $\widetilde{\phi} = (\phi^*)^{-1}$. Since $H(Y \mid \alpha) = H(Y) - I(Y;\alpha) \le \log M - I(Y;\alpha)$ \cite{cover2006elements}, we have
\begin{equation}
\pi(Y \mid \alpha) \le \widetilde{\phi}(\log_2 M - I(Y;\alpha))\label{eq:info-error-bound}
\end{equation}
where $I(Y;\alpha)$ is measured in bits.

If we take the prediction error of $Y$ from $\alpha$ as a measure of ``mismatch'' between our trained model and the blackbox classifier, \eqref{eq:info-error-bound} bounds this mismatch by the causal effect term in our objective and can serve as a certificate for having sufficient network capacity. For example, in 3-class Fashion MNIST ($M=3$), a value of $I(\alpha;Y)=1.03$ nats as in Figure \ref{fig:tuning-fmnist} results in a bound of $\pi(Y \mid \alpha) \le 0.05$. This translates to a MAP estimator of $Y$ from $\alpha$ having a black-box output prediction error of less than $5\%$, or that the causal factors can explain at least $95\%$ of the black-box's behavior. If this prediction accuracy is satisfactory, then the capacity of the generator $g$ is sufficient to learn appropriate latent factors and their mapping to the data space. If this prediction accuracy is not satisfactory, a class $G$ of generative models $g$ with higher capacity can be used. This will provide the model with more flexibility to optimize $I(\alpha;Y)$ and reduce prediction error.

\subsection{Empirical results}

\begin{figure}
    \centering
    \begin{tabular}{ccc}
        \includegraphics[width=0.48\textwidth]{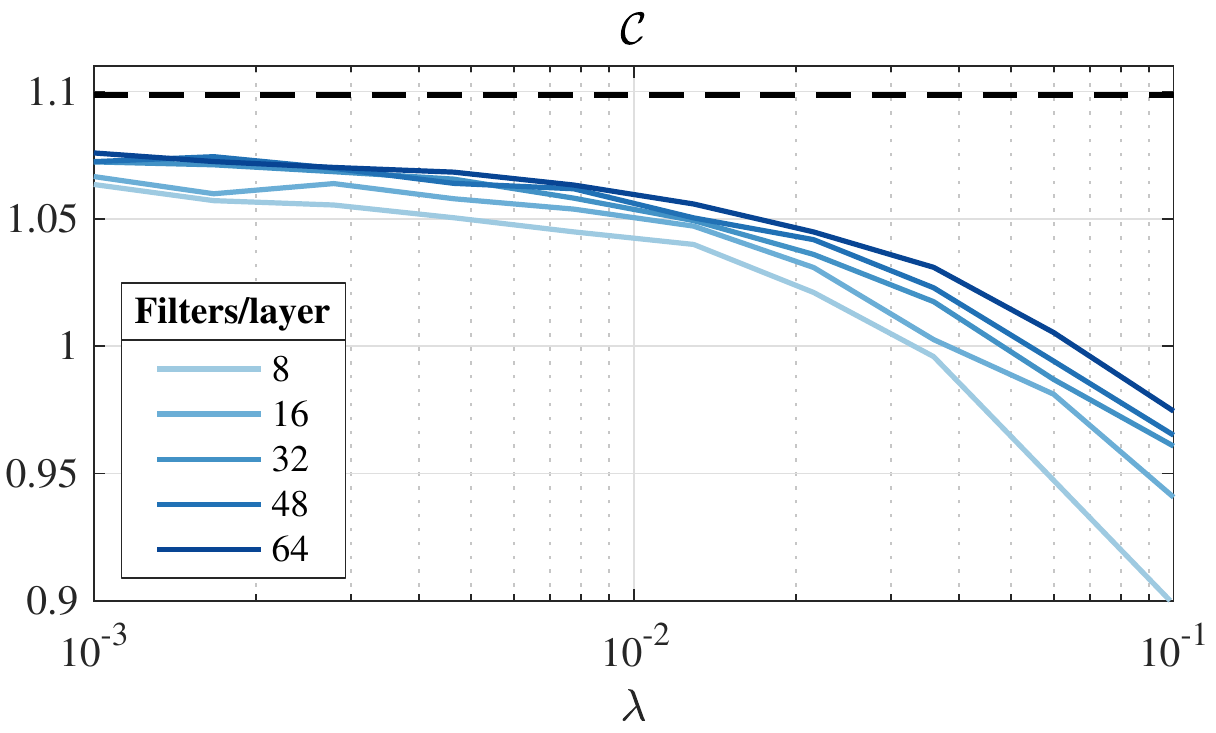} & \quad & \includegraphics[width=0.48\textwidth]{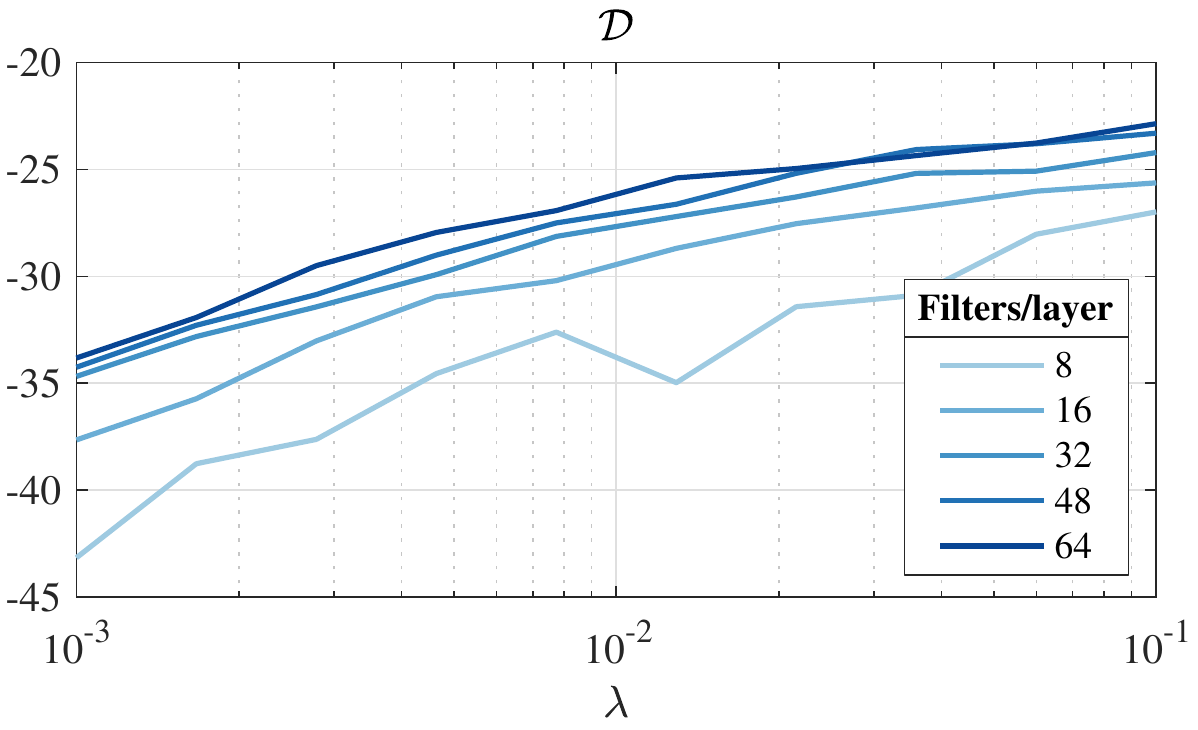} \\ ~~(a) & \quad & ~~(b)
    \end{tabular}
    \caption{Post-training value of the (a) causal effect and (b) data fidelity terms in the objective \eqref{eq:objective} for various capacities of VAE. The capacity is modified by changing the number of convolutional filters in each layer.}
    \label{fig:vae-capacity}
\end{figure}

The drawback of a VAE with insufficient capacity can be seen in Figure \ref{fig:vae-capacity}, which shows the causal effect and data fidelity terms of the objective \eqref{eq:objective} as the VAE capacity and tuning parameter $\lambda$ are modified. The VAE in each trial, which is applied to explain the 3-class Fashion-MNIST classifier considered in the quantitative experiments of Section \ref{sec:vae}, uses the architecture described in Table \ref{tab:mnist-vae-details} with $K = 2$ and $L = 4$ but with a variable number of convolutional filters in each layer of the encoder and decoder (see Table \ref{tab:vae-capacity-nparams}). The values of $\mathcal{C}$ and $\mathcal{D}$ reported in Figure \ref{fig:vae-capacity} are the average values in the last 50 training steps for each model. The dotted line in Figure \ref{fig:vae-capacity} represents the maximum achievable value of $I(\alpha; Y)$ in this three class setting, $\log(3) \approx 1.1$ nats.

\begin{table}
    \centering
    \begin{tabular}{|ccc|} 
        \hline
        Filters per convolutional layer & Encoder parameters & Decoder parameters \\ \hline
        8 & 6,916 & 4,937 \\
        16 & 17,916 & 13,969 \\
        32 & 52,204 & 44,321 \\
        48 & 102,876 & 91,057 \\
        64 & 169,932 & 154,177 \\
        \hline
    \end{tabular}
    \caption{Number of VAE parameters when $K+L = 6$.}
    \label{tab:vae-capacity-nparams}
\end{table}

As discussed in Section \ref{sec:methods/training-procedure}, the tuning parameter $\lambda$ dictates the trade-off between the objective's causal effect term $\mathcal{C}$ and data fidelity term $\mathcal{D}$. When the number of filters per layer is small, however, the model has insufficient capacity to simultaneously achieve a satisfactory value of both $\mathcal{C}$ and $\mathcal{D}$.

\begin{figure}
    \centering
    \begin{tabular}{ccc}
        \includegraphics[]{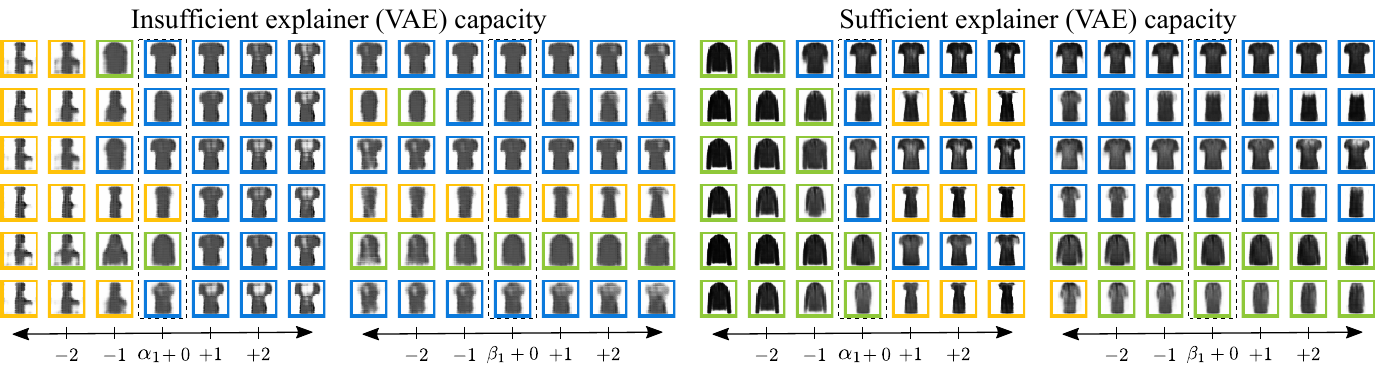} \\ (a) \hspace{1.2in} (b) \hspace{1.4in} (c) \hspace{1.2in} (d)
    \end{tabular}
    \caption{Global explanations with $\lambda \approx 0.013$ and varying VAE model capacity. (a--b) 8 filters per convolutional layer, defining a VAE with insufficient capacity to represent the data distribution. (c--d) 64 filters per convolutional layer, defining a VAE with sufficient capacity to represent the data distribution.}
    \label{fig:vae-capacity-explanations}
\end{figure}

Figure \ref{fig:vae-capacity-explanations} shows partial resulting explanations generated by an explainer with insufficient capacity (8 filters per convolutional layer; Figure \ref{fig:vae-capacity-explanations}(a--b)). Although the causal and noncausal factors do indeed roughly correspond to classifier-relevant and classifier-irrelevant data aspects in the sense that changing $\alpha_1$, but not $\beta_1$, produces changes in the classifier output, the effect of the model's limited ability to represent the data distribution is evident in the weak correspondence of the generated samples to training samples. Meanwhile, the same explanation generated by an explainer with sufficient capacity (64 filters per convolutional layer; Figure \ref{fig:vae-capacity-explanations}(c--d)) shows both effectively disentangled classifier-relevant/irrelevant data aspects and generated samples that appear to lie in the training data distribution.

\end{document}